\documentclass[sn-mathphys,Numbered]{sn-jnl}

\usepackage{graphicx}%
\usepackage{multirow}%
\usepackage{amsmath,amssymb,amsfonts}%
\usepackage{amsthm}%
\usepackage{mathrsfs}%
\usepackage[title]{appendix}%
\usepackage[dvipsnames]{xcolor}%
\usepackage{textcomp}%
\usepackage{manyfoot,ulem}%
\usepackage{booktabs}%
\usepackage{algorithm}%
\usepackage{algorithmicx}%
\usepackage{algpseudocode}%
\usepackage{listings}
\usepackage{svg}
\usepackage{fullpage}

\usepackage{mysty}

\title[]{On the Sample Complexity of Learning for Blind Inverse Problems}

\author*[1]{\fnm{Nathan} \sur{Buskulic}}\email{nathan.buskulic@proton.me}
\author[1,3]{\fnm{Luca} \sur{Calatroni}}\email{luca.calatroni@unige.it}
\author[1,3]{\fnm{Lorenzo} \sur{Rosasco}}\email{lorenzo.rosasco@unige.it}
\author[2]{\fnm{Silvia} \sur{Villa}}\email{silvia.villa@unige.it}

\affil[1]{\orgname{MaLGa - DIBRIS -  Università degli studi di Genova}, \country{Italy}}
\affil[2]{\orgname{MaLGa - DIMA -  Università degli studi di Genova}, \country{Italy}}
\affil[3]{\orgname{Italian Institute of Technology, Genova,} \country{Italy}}      

\begin{document}

\abstract{
%
Blind inverse problems arise in many experimental settings where both the signal of interest and the forward operator are (partially) unknown. In this context, methods developed for the non-blind case cannot be adapted in a straightforward manner due to identifiability issues and symmetric solutions inherent to the blind setting. Recently, data-driven approaches have been proposed to address such problems, demonstrating strong empirical performance and adaptability. However, these methods often lack interpretability and are not supported by theoretical guarantees, limiting their reliability in domains such as applied imaging where a blind approach often relates to a calibration of the acquisition device. In this work, we shed light on learning in blind inverse problems within the insightful framework of Linear Minimum Mean Square Estimators (LMMSEs). We provide a theoretical analysis, deriving closed-form expressions for optimal estimators and extending classical recovery results to the blind setting. In particular, we establish equivalences with tailored Tikhonov-regularized formulations, where the regularization structure depends explicitly on the distributions of the unknown signal, of the noise, and of the random forward operator. We also show how the reconstruction error converges as the noise and the randomness of the operator diminish when we use a source condition assumption. Furthermore, we derive finite-sample error bounds that characterize the performance of the learned estimators as a function of the noise level, problem conditioning, and number of available samples. These bounds explicitly quantify the impact of operator randomness and show explicitly the dependence of the associated convergence rates to this randomness factors. Finally, we validate our theoretical findings through illustrative exemplar numerical experiments that confirm the predicted convergence behavior.
}


\maketitle
      
\section{Introduction}



Blind inverse problems arise in many scientific and engineering applications, where the forward operator describing the - often linear - acquisition process is unknown or only partially known. An illustrative example happens regularly in experimental sciences where the calibration of the measuring device (microscope, telescope,\dots) is never perfectly known. This lack of operator knowledge fundamentally limits the quality of the reconstructed signal of interest and necessitates the development of principled and computationally efficient solution methods. Among the various instances of blind inverse problems, blind image deconvolution is arguably one of the most extensively studied, across both image processing and inverse problems communities; see, for example, \cite{Kundur1996,Bertero_1998,Ming2003,Levin2009} for surveys and representative contributions.

Classical reconstruction methods are based on retrieving a \textit{maximum a posteriori} solution where the prior functionals of the signal and operator structure are handcrafted based on expert modeling~\cite{benning_modern_2018}, which makes them difficult to use in general contexts. 
On the other hand, over the past decade, learning-based methods have achieved state-of-the-art (empirical) performance in several applied inverse problems contexts upon suitable supervised/unsupervised training and availability of training data~\cite{arridge_solving_2019}. However, such data-driven approaches are often difficult to interpret and typically lack rigorous theoretical guarantees, which are essential in safety-critical applications such as biomedical imaging \cite{Ongie2020,Knoll2020}. Most learning-based approaches to inverse problems are built upon highly nonlinear models, which often hinder interpretability and theoretical analysis. In this work, we adopt a complementary viewpoint and focus on learned \textit{linear} estimators, whose does include, for instance, linear neural networks \cite{saxe2014,Gunasekar2018,woodworth_kernel_2020}. This restriction has the potential to enable a complete theoretical treatment while retaining sufficient modeling power to investigate fundamental aspects of blind inverse problems~\cite{alberti2021learning, chirinos2024learning}.

\smallskip

Given two finite-dimensional Hilbert spaces $X, Y$, we consider in this paper blind linear inverse problems, in which the goal is to recover an unknown signal $\xvc \in X$, assumed to be drawn from a distribution $\pi_\xv$, from indirect and noisy observations $\yv \in Y$, when the forward operator $\fop\in \mathcal{L}(X,Y)$ is unknown or only partially known. The observation model reads
\begin{align}\label{eq:prob_inv}
  \yv = \fop\xvc + \veps,
\end{align}
where $\veps$ denotes a centered Gaussian random vector,
$\veps \sim \mathcal{N}(0,\beta \Id)$, where $\beta > 0$ denotes the noise level and $\Id:Y\rightarrow Y $ the identity map.
To account for the blind setting, the forward operator $\fop$ is treated as a random matrix drawn from a distribution
$\pi_\fop$ and is not directly observed; this contrasts with the classical non-blind formulation of~\eqref{eq:prob_inv}, where $\fop$ is assumed
to be known. Even when one has full knowledge of the problem, the potential ill-posedness of $\fop$ alongside noise perturbations can make these inverse problems extremely challenging to solve. Throughout this work, we restrict our attention to the finite-dimensional case, with
$X = \R^n$, $Y = \R^m$, and $\fop \in \R^{m \times n}$, leaving the generalization to infinite-dimensional case for future work.

Classical variational approaches computing  stable solutions of ill-posed inverse problems of retrieving $\xv$ from $\yv$ are based on the minimization of suitable energy functionals. 
In the non-blind case, this amounts to solving
\begin{align}\label{eq:energy_min_variational}
  \Argmin_{\xv} \; D(\fop \xv, \yv) + J(\xv),
\end{align}
where $D : Y \times Y \to \R^+$ is a data-fidelity term measuring consistency between the (known) model $\fop$ and the observations, and 
$J : X \to \R^+$ is a regularization functional that penalizes undesirable structures in $\xv$ and mitigates the ill-posedness of the problem. Such variational formulations can be derived from a Bayesian interpretation: for appropriate choices of $D$ and $J$, solving \eqref{eq:energy_min_variational} is equivalent to computing a maximum a posteriori (MAP) estimate, namely
\begin{equation*}
\Argmin_{\xv} \; -\log p(\yv \mid \xv) - \log p(\xv),
\end{equation*}
where the link with~\eqref{eq:energy_min_variational} is obtained if the likelihood and prior are taken as negative exponentials~\cite{geman1984stochastic}. These approaches can naturally be extended to the blind setting by also considering prior information over the distribution of $\fop$. That is, if we consider $\xv$ and $\fop$ to be independent we can write our MAP formulation as
\begin{equation*}
\Argmin_{\xv,\fop} \; -\log p(\yv \mid \xv, \fop) - \log p(\xv) - \log p(\fop).
\end{equation*}
Here the prior over $\fop$ could be taken as a first step as simply a domain constraint, such as for convolution operators that we often require to live on the simplex. Of course, if one has access to more information on the structure of $\fop$, it can be integrated in this formulation.

While natural, this blind MAP formulation has been shown to suffer from identifiability problems in some contexts~\cite{Fergus2006,Rameshan2012,Perrone2014}, even when using advanced diffusion-based prior~\cite{nguyen2025diffusion}.
In the context of blind deconvolution of natural images with no noise, for instance, the no-blur solution, where the optimal convolution kernel is a Dirac and where the optimal $\xv = \yv$, turns out indeed to be the global MAP solution. Similar data-driven methods~\cite{chihaouiBlindImageRestoration2024, chungParallelDiffusionModels2022, li2024blinddiff, sanghvi2024kernel, laroche2024fast} need to either add new regularization terms to work in practice, without any kind of theoretical motivation, or optimize a marginalized posterior rather than the joint posterior.

Given such poor adaptation of such reconstruction procedures to the blind case, it is natural to explore alternative estimators that might be more robust in this setting. A good candidate is the Minimum Mean Square Estimator (MMSE), that is given by 
\begin{align*}
\Argmin_{\widehat{\xv}} \Expect{}{\norm{\widehat{\xv} - \xv }^2\mid \yv} = \Expect{}{\xv\mid\yv}.
\end{align*}
It is known to be more stable~\cite{freeman1995bayesian}, usually at the cost of being computationally more demanding as one has to integrate over the posterior distribution $p(\xv\vert\yv)$ or, when sample approximations are used, draw many samples to compute sufficiently precise approximations. However, in the simplified case where a Linear MMSE (LMMSE) estimator $\Lv:Y\rightarrow X$ is considered, closed-form formulas can be found, see, e.g. \cite{Kay1998} for a survey. This has been extensively studied for decades in the non-blind case, being such estimators at the very basis of the Wiener-Kolmogorov ~\cite{wiener1964extrapolation} and Kalman filters~\cite{kalmanNewApproachLinear1960} (see~\cite{kailath2001linear} for a review over linear estimator theory). 

Interestingly, LMMSE estimators can be characterized as solutions of quadratic problems with a data-fidelity term $D$ corresponding to the mean squared error and the regularization term $J$ being a Tikhonov type; see \cite{Kay1998,alberti2021learning} for the finite- and infinite-dimensional settings, respectively. In particular, the corresponding regularization takes the form
$
J(\xv) = \|\Mv (\xv-\xv_0)\|^2,
$
with a suitable linear operator $\Mv : X \to Z$ and vector $\xv_0 \in X$. 
From this perspective, constructing or learning an LMMSE estimator amounts to identifying an optimal Tikhonov regularization operator $\Mv$ for the inverse problem under consideration. In \cite{alberti2021learning} and \cite{chirinos2024learning} this problem has been studied in the non-blind setting using datasets of signal–observation pairs and empirical risk minimization to compute optimal operators $\Mv$. Furthermore, the recent work~\cite{banert2025noise} provides an in-depth analysis of the cases where the noise covariance is unknown, and the performance gap that appears due to this lack of information in the solutions of the corresponding Tikhonov-like problems.

\medskip

Although a substantial body of work exists for the non-blind case, the extension of these results to the blind setting, including the computation of LMMSE estimators, has not yet been systematically addressed as the community has focused on MAP-type approaches or on highly non-linear learning based methods, bypassing the study of linear estimators for this setting.

\paragraph{Contributions.}
In this paper, we provide a systematic study of linear minimum mean square error (LMMSE) estimation in the context of blind inverse problems. Our main contributions are twofold:
\begin{enumerate}[label=(\roman*)]
\item We derive explicit expressions of the LMMSE estimators for both the signal and the unknown operator and study how the knowledge of one affects the other. In either case, we identify the associated Tikhonov-regularized variational formulation, which serves as the basis for discussing joint estimation strategies from a variational (Tikhonov-like) perspective.
\item We address the problem of learning the LMMSE signal estimator from data in the blind setting. Under a H\"older-type source condition, we show that the reconstruction error of a regularized empirical LMMSE signal estimator can be bounded, in expectation, by the sum of an approximation error—depending on the noise level and the regularity of the problem—and a sampling error—depending on the conditioning of the problem and the number of available samples. For each error term, we provide theoretical results that characterize their behavior across different regimes.
\end{enumerate}
We complement our theoretical analysis with exemplar one-dimensional numerical experiments that validate the proposed estimators and illustrate the predicted error regimes.

\subsection{Notations}

%
Let $(\Omega, \mathcal{A}, \mathbb{P})$ be a probability space. Let $\zv:\Omega\rightarrow\R^d$ be a $d$-dimensional random vector.
We denote the expected value of a random vector $\zv\in\R^d$ as $\Expect{}{\zv} = \theta_\zv$ and the expected value of a random matrix $\Mv \in \R^{m\times n}$ as $\Expect{}{\Mv} = \Theta_\Mv$. We also write the cross-covariance matrix of two random vectors $\xv\in\R^n$ and $\yv\in\R^m$ by $\Cxy\in\R^{n\times m} = \Expect{}{\pa{\xv - \Expect{}{\xv}}\pa{\yv - \Expect{}{\yv}}\tp}$, and thus the notation $\Cxx\in\R^{n\times n}$ denote the covariance matrix of the random vector $\xv$. We denote by $\Vari{\zv}$ the vector (or matrix) containing the variances of each entry of the random vector $\zv$. We denote by $\CV^2(z) = \frac{\Vari{z}}{\Expect{}{z}^2} \in \R^+$ the squared coefficient of variations of the random variable $z$. We say that $a\lesssim b$ if $a\leq cb$ with $c>0$ a fixed constant. For a matrix $\Mv$, we say that $\Mv_{i,j}$ is the entry of row $i$ and column $j$, and we define the $i$-th row of $\Mv$ as $\Mv_i$. we denote by $\vecmat{\Mv}$ the vectorized matrix where we concatenate the rows of the matrix in lexicographic order. When we write $\Expect{}{\cdot}$ or $\pi_\cdot$, the law to be considered is the one of the random quantities of the input.

\section{LMMSE for blind inverse problems}  \label{sec:LLMSE_blind}

The notion of optimal (in the MSE sense) linear estimators is very useful as it can be understood in a very general sense. Consider two random vectors $\zv\in\R^n$ and $\wv\in\R^m$ drawn from distributions $\pi_\zv$ and $\pi_\wv$ which have finite first and second moments. Then, if one wishes to estimate $\zv$ from $\wv$, the LMMSE estimator uniquely defined by $(\Lv^*,\bv^*) \in \R^{n\times m} \times \R^n$ is given by a closed-form formula~ \cite[Equation 12.20]{Kay1998}:
\begin{align}\label{eq:LMMSE}\tag{LMMSE}
\hat{\zv} = \Lv^*\wv + \bv^* &= \Cv_{\zv\wv}\Cv_{\wv\wv}\inv\pa{\wv - \Expect{}{\wv}} + \Expect{}{\zv}\\
 \text{with }  \pa{\Lv^*,\bv^*} &= \Argmin_{\Lv\in\R^{n\times m}, \bv\in \R^n} \Expect{\zv,\wv}{\norm{\zv - \pa{\Lv\wv + \bv}}^2}. \nonumber 
\end{align}
The fact that this holds for any kind of relationship between $\zv$ and $\wv$ is the reason why the LMMSE or its variants have been so important over many fields~\cite{kailath2003view, anderson2005optimal}.

In the case of linear non-blind inverse problems ~\eqref{eq:prob_inv} where the forward operator $\fop:\R^n\rightarrow\R^m$ is known, we have a clear relationship between $\zv = \xv\in\R^n$ and $\wv = \yv \in \R^m$, as $\yv = \fop\xv+\veps$, which allows us to make explicit $\Cxy$ and $\Cyy$ as we wish to recover $\xv$ from $\yv$. In that context, if the noise is centered and independent from the signal, it is known~\cite[Theorem 12.1]{Kay1998} that the LMMSE is given by 
\begin{align}\label{eq:LMMSE-D}\tag{LMMSE-D}
\hat{\xv}_{\text{LMMSE}} = \Expect{}{\xvc} + 
\Cxx\fop\tp \pa{\fop 
	\Cxx\fop\tp + \Ceps}\inv \pa{\yv - \Expect{}{\yv}
}.
\end{align}
Note here, that the invertibility of $\Cyy = \fop \Cxx\fop\tp + \Ceps$ is generally not an issue. The matrix $\Ceps$ is often full-rank, or if it is not, one can use the corresponding pseudo-inverse thanks to the orthogonality principle. In the following, we will consider a regularized version of the LMMSE, which guarantees invertibility in general.

While~\eqref{eq:LMMSE-D} in non-blind case has been known for a long time, no extension, to the best of our knowledge, to the blind case where we have that $\fop$ is random exists in the literature. In the following we will see how one can extend the framework of LMMSE estimation to the blind  setting to recover either $\xvc$ or $\fop$ or both.


\subsection*{General assumptions}
For the remaining of this work, we will consider the following general assumptions:

\begin{enumerate}[label=(A-\arabic*), itemindent=3em, partopsep=1em, itemsep=0.5em]
\item \label{ass:mutually_indep} $\fop\sim\pi_\fop$, $\xv\sim\pi_\xv$ and $\veps\sim\pi_{\veps}$ are mutually independent. 
\item \label{ass:finite_moments} $\fop\sim\pi_\fop$, $\xv\sim\pi_\xv$ and $\veps\sim\pi_{\veps}$ have finite first and second moments.
\item \label{ass:veps_centered} $\pi_{\veps}$ is a centered distribution with $\Ceps = \beta\Id$ for $\beta\in\R_*^+$ and $\Id\in\R^{m\times m}$ the identity matrix
\item \label{ass:fop_share_sing_vec} Any $\fop\sim\pi_\fop$ can be decomposed as $\fop=\Uv\Dv\Vv$ with $\Dv\in\R^{m\times n}$ a diagonal-like matrix and $\Uv\in\R^{m\times m}$ and $\Vv\in\R^{n\times n}$ two fixed orthonormal matrices. Furthermore, we consider that each singular value $\varsigma_i \sim \pi_{\varsigma_i}$ of $\fop$  with $\pi_{\varsigma_i}$ having finite first and second moments and $\Expect{}{\varsigma_i} > 0$.

\item \label{ass:Cxx_Cyy}  $\norm{\Cyy} \geq \norm{\Cxx}$ with $\norm{\cdot}$ the operator norm.
\end{enumerate}

Assumption~\ref{ass:mutually_indep} is standard and realistic, as for many applications, the uncertainty over the operator is independent from the observed signal. For example, the uncertainty over the convolution kernel of a telescope (unknown meteorological conditions) is independent from the part of the sky observed. The independence to the noise distribution is also very standard and corresponds to standard Gaussian noise assumptions. Assumption~\ref{ass:finite_moments} is necessary to the existence of our estimators and is  standard in inverse problems where $\fop$ is not random. Assumption~\ref{ass:veps_centered} is here used for simplicity of the presentation, but does not impede the generality of our results.  
Assumption~\ref{ass:fop_share_sing_vec} will be necessary only to derive  approximation error results. It ensures that all forward operators drawn from $\pi_\fop$ share the same singular vectors. While seeming like a limiting assumption, we note that many real-world problems are based on specific types of operators which share the same basis. For example, the elements of the set of blurring operators are all diagonalizable on a Fourrier basis. 
 Finally,~\ref{ass:Cxx_Cyy} is made for simplicity in the computation of some constant in the theorems but can be removed easily at the cost of slightly less clear constants.

\subsection{LMMSE to recover $\xvc$}

Here we consider the blind setting where given the relation $\yv = \fop\xv + \veps$, $\xv$ is unknown and $\fop$ is random and let us derive an LMMSE estimator to recover $\xvc$ from $\yv$. Looking back at the general~\eqref{eq:LMMSE}, we need to derive the expression of $\Cxy$ and $\Cyy$ in this blind setting since $\Lv^* = \Cxy\Cyy\inv$. Let us start by observing that $\Cxy$ can be expressed as
\begin{align*}
\Cxy &= \Expect{}{\pa{\xvc - \mux}\pa{\fop\xv + \veps - \Theta_\fop\mux}\tp} = \Expect{}{\xvc\xvc\tp\fop\tp} - \mux\mux\tp\Theta_\fop\tp\\
&= \pa{\Cxx + \mux\mux\tp}\Theta_\fop\tp - \mux\mux\tp\Theta_\fop\tp = \Cxx\Theta_\fop\tp
\end{align*}
where we used that $\fop$ and $\xvc$ are independent, and $\veps$ is zero-mean.

The derivation of $\Cyy$ requires more attention. We first observe that since $\veps$ is centered
\begin{align*}
\Cyy &= \Expect{}{\pa{\yv - \theta_\yv}\pa{\yv - \theta_\yv}\tp} = \Expect{}{\fop\xvc\xvc\tp\fop\tp} - \theta_\yv\theta_\yv\tp + \Ceps.
\end{align*}

Now, if $\fop$ is not random, this formulation coincides with to~\eqref{eq:LMMSE-D}. To understand how it changes for the blind case, let us look at one element of $\Expect{}{\fop\xvc\xvc\tp\fop\tp}$:
\begin{align*}
\pa{\Expect{}{\fop\xvc\xvc\tp\fop\tp}}_{i,j} &= \sum_{k=1}^n\sum_{q=1}^n \Expect{}{\fop_{i,k}\fop_{j,q}} \Expect{}{\xvc_k\xvc_q}\\
&= \sum_{k=1}^n\sum_{q=1}^n \pa{\Cv_{\fop_i\fop_j} + \theta_{\fop_i}\theta_{\fop_j}\tp}_{k,q} \pa{\Cxx + \mux\mux\tp}_{k,q}\\
&= \sum_{k=1}^n\sum_{q=1}^n \pa{\Cv_{\fop_i\fop_j} \odot \Cxx}_{k,q} + \theta_{\fop_i}\tp\Cxx\theta_{\fop_j} + \mux\tp\Cv_{\fop_i\fop_j}\mux + \theta_{\yv_i}\theta_{\yv_j}.
\end{align*}

Let us define the matrix $\Dv\in\R^{m\times m}$ such that $\Dv_{i,j} = \sum_{k=1}^n\sum_{q=1}^n \pa{\Cv_{\fop_i\fop_j} \odot \Cxx}_{k,q}$, and let $\Caa\in\R^{mn\times mn}$ be the covariance matrix of $\av = \vecmat{\fop} \in \R^{mn}$. With these element defined, we can express $\Cyy$  as
\begin{align}\label{eq:Cyy_blind}
\Cyy = \Theta_\fop\Cxx\Theta_\fop\tp +  \pa{\Id_m \otimes \theta_\xv\tp}\Caa\pa{\Id_m \otimes \theta_\xv\tp}\tp + \Dv + \Ceps.
\end{align}

A few remarks are in order. First, we see that now $\Cyy$ is the sum of four terms. The first three components represent the interactions between the randomness of $\xvc$ and the one of $\fop$. In particular, the first term represents how the covariance of $\xvc$ interacts with the mean of $\fop$; this is the term that appear in~\eqref{eq:LMMSE-D} when $\fop$ is not random and where $\Theta_\fop = \fop$. The second term represents the interaction between the covariance of $\fop$ with the mean of $\xvc$. Such term  disappears if $\fop$ is deterministic or if $\xvc$ is centered. Finally, the third term $\Dv$ represent the interaction between the two covariances of $\xv$ and $\fop$.  The last term (which does not change with respect to the non-blind case) is the noise, which we need to take into account the same way as before.

While the expression of $\Dv$ is general, under some special cases $\Dv$ simplifies:\\
\textbf{Independent rows}: If $\fop$ has independent rows, then, when $i \neq j$, $\Cv_{\fop_i \fop_j} = 0$ which means that $\Dv$ is diagonal and  the elements of the diagonal are a linear combination of the elements of $\Cxx$ and $\Cv_{\fop_i \fop_i}$.\\
\textbf{Independent columns}: In the case where $\fop$ has independent columns, all the matrices $\Cv_{\fop_i \fop_j}$ become diagonal, and $\Dv$ depends only on $\Vari{\xvc}$ but not anymore on the entries of $\Cxx$.\\ 
\textbf{Independent entries}: When $\fop$ has independent entries, $\Dv$ becomes simply $\Dv = \diag{\Vari{\fop}\Vari{\xvc}}$. In that case, only the variances of the elements of $\fop$ and $\xvc$ matters.

\smallskip


Having defined $\Cxy$ and $\Cyy$, we finally write down the LMMSE signal estimator in the blind case:
\begin{align}\label{eq:LMMSE-B}\tag{LMMSE-$\xv$}
\hat{\xv} = \theta_\xv + \Cxx\Theta_\fop\tp\pa{\Theta_\fop\Cxx\Theta_\fop\tp +  \pa{\Id_m \otimes \theta_\xv\tp}\Caa\pa{\Id_m \otimes \theta_\xv\tp}\tp + \Dv + \Ceps}\inv\pa{\yv - \theta_{\yv}}
\end{align} 

An important characterization of such estimator is given in terms of the solution of a Tikhonov regularized problem, see~\cite[Theorem 12.1]{Kay1998} for the non-blind case. 
\begin{lemma}[\eqref{eq:LMMSE-B} as Tikhonov solution]\label{th:LMMSE-TikhG}
Assume that $\Cxx$ and $\Cyy$ are invertible, then the estimate given in \eqref{eq:LMMSE-B} is the solution of the generalized Tikhonov problem:
\begin{align}\label{eq:Tikh-sig}\tag{Tikh-$\xv$}
\Argmin_{\zv_1} \norm{\Theta_\fop \zv_1 - \yv}^2_{\Cp\inv} + \norm{\zv_1 - \mux}^2_{\Cxx\inv}
\end{align}
where $\norm{\zv}^2_\Mv = \zv\tp\Mv\zv$ and $\Cp = \pa{\Id_m \otimes \theta_\xv\tp}\Caa\pa{\Id_m \otimes \theta_\xv\tp}\tp + \Dv + \Ceps$. 
\end{lemma}

\begin{proof}
See Section~\ref{proof:lmmse-tikhG}.
\end{proof}

A special case of this result is when $\xvc \sim \pi_{\xv} = \mathcal{N}(0,\sigma_x^2\Id)$, $\veps \sim \pi_{\veps} = \mathcal{N}(0,\sigma_{\veps}^2\Id)$ and when $\fop$ has independent columns and $\forall i$, $\sum_{j=1}^m \sigma_{\fop_{ij}}^2 = c_\fop^2$ a constant value. In that case, we have that $\Cp = (c_\fop^2 \sigma_{x}^2 + \sigma_{\veps}^2) \Id$ and $\Cxx = \sigma_{x}^2 \Id$, which means that the LMMSE corresponds to the solution of a classical Tikhonov problem with regularization parameter $\lambda = c_\fop^2 + \frac{\sigma_{\veps}^2}{\sigma_{x}^2}$, which shows that the optimal parameter in this case takes into account not only the amount of noise w.r.t. the prior information on $\xv$, but also the randomness of $\fop$.


\subsection{LMMSE to recover $\fop$}

Similar as above, we can address the question of constructing a LMMSE estimator of $\fop$ given $\yv$. As for the estimation of $\xvc$, we thus look for an optimal linear estimator $\Lv_{\fop} \in \R^{mn \times m}$ and bias $\bv_\fop \in \R^{mn}$ such that:
\[
\Lv_\fop, \bv_\fop \in \Argmin_{\Lv,\bv} \Expect{}{\norm{\Lv\yv - \mathbf{a} + \bv}^2}
\] 
with $\av = \vecmat{\fop}$. To perform the matrix recovery, we are minimizing over the Frobenius norm. By~\eqref{eq:LMMSE}, we know that $\Lv_\fop$ takes the form $\Lv_\fop=\Cay\Cyy\inv$. with $\Cyy$ as in~\eqref{eq:Cyy_blind} and $\bv_\fop = \Expect{}{\av} - \Lv_\fop\Expect{}{\yv}$. Let us compute $\Cay$. To do so,  we will first provide an observation that will be very useful for the rest of this section, which is:
\begin{align}\label{eq:equivalence_yv_kron}
\yv = \fop \xvc + \veps = \pa{\Id_m \otimes \xvc\tp}\av + \veps
\end{align}
where $\Id_m$ is the identity matrix of size $m\times m$. We can thus now easily express $\Cay$:
\begin{align*}
\Cay &= \Expect{}{\av\yv\tp} - \theta_\av\theta_\yv\tp = \Expect{}{\av\av\tp}\pa{\Id_m \otimes \theta_\xv\tp}\tp - \theta_\av\theta_\yv\tp\\
&= \pa{\Caa + \theta_\av\theta_\av\tp} \pa{\Id_m \otimes \theta_\xv\tp}\tp - \theta_\av\theta_\av\tp \pa{\Id_m \otimes \theta_\xv\tp}\tp\\
&= \Caa\pa{\Id_m \otimes \theta_\xv\tp}\tp
\end{align*}
where we used the independence between $\fop$, $\xvc$ and $\veps$ and that $\Expect{}{\veps}=0$. Thus we can express the optimal linear estimator of $\av$ recalling~\eqref{eq:Cyy_blind} to get:
\begin{align}\label{eq:LMMSE_fop}\tag{LMMSE-$\fop$}
\hat{\av} = \Caa\pa{\Id_m \otimes \theta_\xv\tp}\tp\pa{\mufop \Cxx \mufop\tp + \pa{\Id_m \otimes \theta_\xv\tp}\Caa\pa{\Id_m \otimes \theta_\xv\tp}\tp + \Dv + \Ceps}\inv \pa{\yv - \theta_{\yv}} + \theta_\av.
\end{align}
As for~\eqref{eq:LMMSE-B}, we can express~\eqref{eq:LMMSE_fop} as the solution of a Tikhonov regularized problem akin to the result of Lemma~\ref{th:LMMSE-TikhG}:

\begin{lemma}\label{th:LMMSE_fop_Tikh}
Assume that $\Caa$ and $\Cyy$ are invertible, then the estimate given in \eqref{eq:LMMSE_fop} is the solution of the generalized Tikhonov problem:
\begin{align*}
\Argmin_{\zv_2} \norm{\pa{\Id_m \otimes \theta_\xv\tp}\zv_2 - \yv}^2_{\Cv_{p_2}\inv} + \norm{\zv_2 - \theta_{\av}}^2_{\Caa\inv}.
\end{align*}
where $\Cv_{p_2} = \theta_\fop\Cxx\theta_\fop\tp + \Dv+\Ceps$.
\end{lemma}

\begin{proof}
See Section~\ref{proof:lmmse-tikh-op}
\end{proof}

The weights of the the data-fidelity term is similar to the one in Lemma~\ref{th:LMMSE-TikhG}, which is expected as one still has to deal with the noise and the interaction between the uncertainties of $\fop$ and $\xv$ represented by $\Dv$. On the other hand, the weight of the Tikhonov regularization is tailored to enforcing regularization on the operator and not on the signal. As, we are considering the LMMSE that minimizes the Frobenius norm, it is natural to obtain something similar to~\eqref{eq:LMMSE-B}. It would be interesting to consider LMMSE defined in terms of other matrix norms. We leave this question for future work.


\subsection{Remarks on joint LMMSE estimation}

We combine the individual LMMSE estimates~\eqref{eq:LMMSE-B} and~\eqref{eq:LMMSE_fop} into a joint optimization framework of the form
\begin{equation} \label{eq:joint_minimization}
\Argmin_{\zv_1\in\R^n,\zv_2\in\R^mn} 
\norm{\Theta_\fop \zv_1 - \yv}^2_{\Cv_{p}^{-1}}
+ \norm{\pa{\Id_m \otimes \theta_\xv\tp}\zv_2 - \yv}^2_{\Cv_{p_2}\inv}
+ \norm{\zv_1 - \mux}^2_{\Cxx^{-1}}
+ \norm{\zv_2 - \theta_{\av}}^2_{\Caa^{-1}},
\end{equation}
where $\Cv_{p}$ and $\Cv_{p_2}$ are chosen as the covariance matrices appearing in Lemmas~\ref{th:LMMSE-TikhG} and~\ref{th:LMMSE_fop_Tikh}, respectively. In contrast to classical MAP-based formulations, which involve a joint data-fidelity term of the form $\|\fop \xv - \yv\|^2$ leading to intrinsically non-convex optimization problems and potential issues (see, e.g., \cite{Rameshan2012,Perrone2014,nguyen2025diffusion}), the objective \eqref{eq:joint_minimization} relies on a different structure of the data-fidelity term. Specifically, the contribution of each unknown is marginalized by taking expectations with respect to the other variable, resulting in two decoupled quadratic data terms. This decoupling renders \eqref{eq:joint_minimization} jointly convex and allows the minimization to be carried out in a single step, since the estimate of $\xv$ (resp.\ $\fop$) does not depend on the current estimate of $\fop$ (resp.\ $\xv$). Proceeding as in Lemmas~\ref{th:LMMSE-TikhG} and~\ref{th:LMMSE_fop_Tikh}, one can thus show that solving \eqref{eq:joint_minimization} is equivalent to computing the joint linear MMSE estimator
\begin{align*}
\min_{\Lv_\xv,\Lv_\av}
\Expect{}{\norm{\Lv_\xv \yv - \xv + \bv_\xv}^2 + \norm{\Lv_\av \yv - \av + \bv_\fop}^2},
\end{align*}
which minimizes the sum of the MSE associated with the signal and operator estimates where $\bv_\xv = \Expect{}{\xv} - \Lv_\xv\Expect{}{\yv}$ and $\bv_\fop = \Expect{}{\av} - \Lv_\av\Expect{}{\yv}$. By linearity, this joint problem decomposes into two independent LMMSE problems with respect to $\xv$ and $\fop$, which can be solved separately. 

Alternatively, note that by the tower property, the joint risk can equivalently be expressed in terms of conditional risks, e.g.,
\[
\Expect{}{\|\Lv_\xv \yv - \xv\|^2}
=
\Expect{\fop}{\Expect{\xv,\yv\mid\fop}{\|\Lv_\xv \yv - \xv\|^2 }}
\]
This reformulation naturally leads to a nested (bilevel) optimization problem, in which the signal is estimated in a lower-level problem for a fixed forward operator, while the operator itself is identified at the upper level based on the resulting reconstruction, see, e.g., \cite{Hintermueller2015}. 


In the present work, we do not pursue this bilevel formulation further and focus instead on the estimation of the signal for a fixed joint LMMSE structure; exploring fully bilevel formulations for blind LMMSE estimation constitutes a further natural direction of future work.

%



\section{Empirical reconstruction of $\xvc$} \label{sec:LLMSE_empirical}


Now that we completed the LMMSE formulas in the blind setting, it is natural to wonder what kind of theoretical performance guarantees one should expect from such estimators. In order to do this, we will consider a setting where the information from the different probability distributions necessary to build the LMMSE estimators can only be approximated by using a finite number of paired signal-observation samples. That is, we consider to have access to $N\in\mathbb{N}$ i.i.d. samples $(\xv_k,\yv_k)_{k=1}^N$ where $\xv_k\sim \pi_\xv$ and $\yv_k = \fop_k\xv_k + \veps_k$ with $\fop_k \sim \pi_\fop$ and $\veps_k\sim\pi_{\veps}$. From these samples, we build a regularized empirical LMMSE operator $\widehat{\Lv}^\lambda_\xv$ defined by:
\begin{align}\label{eq:approx_LMMSE}
\widehat{\Lv}^\lambda_\xv = \underbrace{\frac{1}{N}\sum_{k=1}^N \pa{\xv_k - \theta_\xv}\pa{\yv_k - \theta_\yv}\tp}_{\approx\Cxy}\underbrace{\pa{\frac{1}{N}\sum_{k=1}^N \pa{\yv_k - \theta_\yv}\pa{\yv_k - \theta_\yv}\tp + \lambda\Id}\inv}_{\approx\Cyy\inv}.
\end{align}
This estimator uses the true first moment $\theta_\xv$ and $\theta_\yv$ of the distribution but we mention in a following remark how to deal with empirical mean estimation. We will note the true regularized estimator $\Lv^\lambda_\xv = \Cxy\pa{\Cyy + \lambda\Id}\inv$.

Note that we regularize the empirical approximation of $\Cyy$ by adding $\lambda\Id$ where $\lambda \in\R^+_*$ parametrize the strength of the regularization. This regularization ensures the invertibility of the matrix and will allow to better balance the classical bias-variance tradeoff  appearing when building such empirical estimators.

The error committed by the empirical LMMSE ( $\Expect{}{\norm{\widehat{\Lv}^\lambda_\xv\yv - \xvc + \bv }^2}$) will be split between an estimation error, that quantifies a suitable distance between the empirical estimator and the ideal one, and an approximation error due to the inherent randomness of the problem, coming from both the noise and the forward operator, alongside the limitation of using a linear estimator (bias). These error sources will be studied and discussed on their own in Section~\ref{sec:bias_error} and Section~\ref{sec:variance_error}, respectively.

The approximation error will be studied under a Hölder source condition, to ensure proper convergence rates of the problem.
\begin{definition}[Hölder source condition with random singular values] \label{def:source_cond_random_sv}
We say that a distribution $\pi_\xv$ obey a source condition if 
\begin{align}
\Cxx = \pa{\Expect{}{\pa{\fop\tp\fop}}}^\alpha,
\end{align}
for some $\alpha >0$.
\end{definition} 
Having this kind of source condition is very common when deriving convergence rates for inverse problems, and is a necessary condition to achieve fast convergence rates. Such a  Hölder source condition for the blind case is a natural extension of the source condition in the non-blind case which often reads as $\Cxx=\pa{\fop\tp\fop}^\alpha$~\cite{knapik2011bayesian}. In both these contexts, $\alpha$ implies a notion of regularity of the solution, where the higher $\alpha$ is, the smoother the solution. For the non blind case, it directly translates to better convergence rates with respect to the noise, but we will see that this changes in the blind setting.

We present in the next theorem our main theoretical error bound balancing the bias-variance trade-off and providing performance guarantees when using an empirical LMMSE in a blind inverse problem setting. 

\begin{theorem}\label{th:main_aprrox_and_sampling}
 Assume that $\norm{\xv - \theta_\xv}^2\leq \rho_\xv$ and $\norm{\yv - \theta_\yv}^2\leq \rho_\yv$ with $\rho_\xv,\rho_\yv \in \R^+ $ two constants such that $\rho_\xv<\infty$ and $\rho_\yv<\infty$ and that $\pi_\xv$ satisfies Definition~\ref{def:source_cond_random_sv} with $\alpha>0$. Let 
$K = \log\pa{n+m}\frac{4\max\pa{\rho_\xv,\rho_\yv}\norm{\Cyy}}{3\norm{\Cxx}^2}$ and $\gamma=\sigmin\pa{\Cyy}$. Then if the number of samples $N>K$ and $\lambda+\gamma \geq 4\sqrt{\frac{K}{N}}$, we have that 
\begin{align}\label{eq:main_th}
\Expect{}{\norm{\widehat{\Lv}^\lambda_\xv\yv - \xvc + \bv }^2} \lesssim m \pa{\beta+\lambda}^{\frac{\alpha}{\alpha+1}} + \sum_i^m \Expect{}{\varsigma_i^2}^{\alpha}\frac{\Vari{\varsigma_i}}{\Expect{}{\varsigma_i^2}} + m\frac{C}{\pa{\gamma+\lambda}N}\pa{1 + \frac{1}{\pa{\gamma+\lambda}^2} + \frac{1}{\gamma+\lambda}}
\end{align}
with probability $1 - 2\pa{n+m}\inv$ where $C$ is a constant depending on $\norm{\Cxy}$ and $K$, and $\bv = \pa{\Id - \widehat{\Lv}^\lambda_\xv\Theta_\fop}\theta_\xv$ the bias given by~\eqref{eq:LMMSE}.
\end{theorem} 
\begin{proof} This result is obtained by decomposing the error in two terms:
\begin{align}\label{eq:proof_main_th}
\Expect{}{\norm{\widehat{\Lv}^\lambda_\xv\yv - \xvc + \bv }^2} \leq  2\Expect{}{\norm{\Lv^\lambda_\xv\yv - \xv + \bv}^2} + 2\Expect{}{\norm{\widehat{\Lv}^\lambda_\xv\yv - \Lv^\lambda_\xv\yv}^2} 
\end{align}
where the first term is an approximation error term controlled through Theorem~\ref{th:error_bound_singular_values} while the second term is a stability term that we control with Corollary~\ref{corollary:sampling_bound}.
\end{proof}

A few remarks are in order. The first one is that we can separate the error from the uncertainties inherent to the problem from the error that comes from using a finite sample approximation of the LMMSE. The first term that depends on the noise level $\beta$ is very classical in the inverse problem literature and is known as the ``convergence rate'' which indicates how the expected error bound evolves when the noise level decreases. However, in the blind case a new term is added to the convergence rate which encodes the error that one can expect given the level of randomness on the forward operator. More precisely, we see in~\eqref{eq:main_th} that the rate depends on the ratio of the variance to the second moment of each singular value of $\fop$, scaled by the second moment to the power of the source condition variable. This relation between ill-posedness and randomness appears natural as high level of randomness on an important singular value will be much more impactful on the forward operator than the same randomness on a small magnitude singular value. 

The last term in~\eqref{eq:proof_main_th} represents the sampling error due to having only a finite number of samples $(\xv_k,\yv_k)_{k=1,\dots,N}$ to approximate the moments of their distributions. As it is classical for such sampling bounds, such term decreases with a rate of $O(1/N)$ and is scaled by some power of $(\gamma + \lambda)\inv$. Notably, for ill-conditioned problems, $\gamma$ will be very small and the regularization weight $\lambda$ will allow to avoid $1/\gamma$ to diverge that would in turn needs for $N$ to be extremely large to compensate the ill-posedness. Thus, the use of  $\lambda$ here controls the scaling ensuring that even with a limited number of samples, the sampling bound stays meaningful.

However, $\lambda$ has a negative impact on the noise convergence rate since $\lambda$ acts a trade-off that needs to be balanced depending on the noise level and the number of available samples for the given ill-posed problem at hand. While an optimal choice of $\lambda$ may exist, this  is dependent on exact information of the source condition parameter $\alpha$, and on $\gamma$, and $\beta$ which are typically not available. We also note that since our theoretical guarantees results are proved in expectation, the best $\lambda$ that optimizes this bound can be different from what is the best $\lambda$ for an individual practical problem which can be search for by cross-validation for example. 

Finally we note that we used for clarity a fixed probability $1 - 2\pa{n+m}\inv$ that depends on the problem dimensions but by changing $K$, one can obtain different high probability bound. The general results are presented and discussed in Theorem~\ref{th:main_sampling_bound}.
In the following, we examine how each error behave in more details.

\subsection{Approximation error under spectral regularization}\label{sec:bias_error}

%
In this section we will study the first term in~\eqref{eq:proof_main_th}, i.e., the approximation error of $\Lv_\xv^\lambda = \Cxy\pa{\Cyy + \lambda\Id}\inv$, that is, the regularized version of the LMMSE signal estimator~\eqref{eq:LMMSE-B}. As mentioned before, we consider this error under the source condition, expressed in Definition~\ref{def:source_cond_random_sv}. Recalling~\ref{ass:fop_share_sing_vec} we further have that all operators share the same singular vectors so that only singular values are drawn from some probability distribution.

We now write the approximation error bound that can be obtained for the blind case under spectral regularity:
\begin{theorem}\label{th:error_bound_singular_values}
Assume that the distribution of signals $\xv\sim\pi_\xv$ obeys~\ref{def:source_cond_random_sv}.  Then, the expected reconstruction error can be bounded as:
\begin{align}
\Expect{}{\norm{\Lv_\xv^\lambda\yv - \xvc + \bv}^2} \leq m \pa{\beta+\lambda}^{\frac{\alpha}{\alpha+1}} +\sum_i^m \Expect{}{\varsigma_i^2}^{\alpha}\frac{\Vari{\varsigma_i}}{\Expect{}{\varsigma_i^2}} \label{eq:approx_bound}
\end{align}
with  $\varsigma_i$ denoting the $i$-th singular value of $\fop$, and $\bv = \pa{\Id - \Lv^\lambda_\xv\Theta_\fop}\theta_\xv$ the bias given by~\eqref{eq:LMMSE}.
\end{theorem}
\begin{proof}
See Section~\ref{subsec:proof_approx}.
\end{proof}

%

The first term in~\eqref{eq:approx_bound}, shows the convergence rate. It is very classical in standard (non-blind) settings and represents the effect that should be expected from the noise given the regularity of the problem.  In this bound, the regularization term $\lambda$ has a detrimental effect as it acts as a perturbation to the true LMMSE in~\eqref{eq:LMMSE-B}. The parameter $\alpha$ of the source condition dictates the rate of convergence with the power $\frac{\alpha}{\alpha+1}$, which gets better (for $\beta + \lambda \leq 1$) the bigger $\alpha$ is.

The second term, is unusual due to the blind setting. It encodes the impact of the randomness of the forward operator.In the non-blind setting, $\Vari{\varsigma_i} = 0$, which cancels the term and entails the classical non-blind results comprising of only the noise term, making~\eqref{eq:approx_bound}  a generalization. The role of $\alpha$ for this second term is different, as $\Expect{}{\varsigma_i^2}^{\alpha}$ can grow exponentially with $\alpha$. Differently from the non-blind case, there is therefore a trade-off between the two terms in the bound given by the $\alpha$ term, whereas usually  in the non blind case, the higher the $\alpha$, the better. Note that more precisely, such
term depends on $\frac{\Vari{\varsigma_i}}{\Expect{}{\varsigma_i^2}}$, where the the impact of the operator uncertainty is given by how fast the variance of each singular value changes with respect to its second moment. In essence, for large singular values, the effect of the variance will be reduced as the term will be dominated by the scaled magnitude $\Expect{}{\varsigma_i^2}^{\alpha}$, while for lower singular values, the effect of the variance will become more important and the dominating term will be $\frac{\Vari{\varsigma_i}}{\Expect{}{\varsigma_i^2}}$. The main limitation here is clearly the exponential effect of $\alpha$ on this second term, which very quickly makes it much bigger than the noise term, even for $\varsigma\approx 1$. This indicates that the main source of error in the this blind setting is due to the randomness of the operator rather than the noise.

Note that the term $\frac{\Vari{\varsigma_i}}{\Expect{}{\varsigma_i^2}}$ is equivalent to $\frac{\CV^2(\varsigma_i)}{1 + \CV^2(\varsigma_i)}$ where $\CV^2$ is the squared coefficient of variation of $\varsigma_i$, which provides a scale independent notion of spread of the distribution. Understanding the rate of convergence of this term that corresponds to understanding the decay of the the coefficient of variations of the singular values of the operator.


\subsection{Sampling bounds} \label{sec:variance_error}

Let us now turn to the estimation error appearing in the statement of Theorem~\ref{th:main_aprrox_and_sampling}. This error is due to the finite number sample average approximating $\Lv_\xv^\lambda$ using~\eqref{eq:approx_LMMSE}.  We state the main sampling theorem allowing to obtain with high probability regularized sample LMMSE estimator arbitrarly close to the LMMSE given enough samples:


\begin{theorem}[Sampling bound of LMMSE] \label{th:main_sampling_bound}
Assume  $\norm{\xv - \theta_\xv}^2\leq \rho_\xv$ and $\norm{\yv - \theta_\yv}^2\leq \rho_\yv$ almost surely with $\rho_\xv<+\infty$ and $\rho_\yv<+\infty$. Then, for any positive $\xi < \lambda+\gamma$, where $\gamma = \sigmin\pa{\Cyy} > 0$, and $d>0$, if the number of samples $N$ is such that
\begin{align}\label{eq:bound_N_to_alpha}
N > \log\pa{\frac{n+m}{d}}\frac{2\max\pa{\rho_\xv,\rho_\yv}\norm{\Cyy}\pa{3+2\xi}}{3\xi^2\norm{\Cxx}^2},
\end{align}
the following sammpling boun holds:
\begin{align}\label{eq:sampling_bound_main}
\Expect{\yv}{\norm{\Lv^\lambda_\xv\yv - \widehat{\Lv}^\lambda_\xv\yv}^2} \leq m\norm{\Cxy}^2\frac{\xi^2}{\gamma+\lambda}\pa{1 + \frac{\pa{1+\xi}^2}{\pa{\gamma+\lambda-\xi}^2} + \frac{2\pa{1+\xi}}{\pa{\gamma+\lambda-\xi}}}
\end{align}
with probability at least $1 - 2d$.
\end{theorem}

\begin{proof}
See Section~\ref{proof:sampling_bound}.
\end{proof}


This theorem entails several interpretations. It depends on three parameters, $\xi$, $\lambda$ and $d$, which can be balanced in various ways. The main choice is $\xi$ (and therefore an associated $\lambda$): smaller values of $\xi$ give sharper bounds, providing that $N$ satisfies condition~\eqref{eq:bound_N_to_alpha}, where the lower $\xi$, the harder the condition is. The parameter $d$ can be adjusted to relax this condition at the cost of a deterioration of the probabilities ensuring the validity of the theorem.

The parameter $\lambda$ also plays an important role. Indeed, $\lambda$ is crucial for ill-posed operators $\fop$ as $\gamma$ is close to zero: in this case it ensures that $\frac{1}{\gamma + \lambda}$ does not explode, which is necessary for the bound to be meaningful. However, $\lambda$ should be chosen carefully as it has a negative impact on the approximation bound previously discussed.


Note that in~\eqref{eq:sampling_bound_main}, the dependence to the number of samples $N$ is implicit. We thus provide a corollary to Theorem~\ref{th:main_sampling_bound} where we specialize the values of $\xi$ and $d$ making such dependence explicit.

\begin{corollary}\label{corollary:sampling_bound}
Let $K = \log\pa{n+m}\frac{4\max\pa{\rho_\xv,\rho_\yv}\norm{\Cyy}}{3\norm{\Cxx}^2}$ . Then if $N> \max\pa{K, \frac{16K}{\pa{\gamma + \lambda}^2}}$ and $\lambda+\gamma \geq 4\sqrt{\frac{K}{N}}$, we have that
\begin{align*}
\Expect{\yv}{\norm{\Lv^\lambda_\xv\yv - \widehat{\Lv}^\lambda_\xv\yv}^2} \leq \frac{mK}{N}\frac{9\norm{\Cxy}^2}{\pa{\gamma+\lambda}}\pa{1 + \frac{16}{\pa{\gamma+\lambda}^2} + \frac{8}{\gamma+\lambda}}
\end{align*}
with probability at least $1 - 2\pa{n+m}\inv$.
\end{corollary}

\begin{proof}
See Section~\ref{proof:corollary_sampling}
\end{proof}

\textbf{Remark.} In the corollary, we chose to use the bound $\xi = c\sqrt{\frac{K}{N}}$ with $c=2$ which in turn requires $N>K$ to hold. In cases where $K$ is large with respect to $N$, one can choose to use bigger values of $c$ which constrain more $\lambda$ but reduce the condition on $N$ with respect to $K$. For example, taking $c=4$, we only need that $N>\frac{K}{26^2}$. This means that even for small number of samples (with respect to $K$), similar sampling bound can be achieved.

\textbf{Remark.} Note that to calculate $\widehat{\Lv}^\lambda_\xv$ we use the true first moments $\theta_\xv$ and $\theta_\yv$ of the distribution to compute our empirical covariances and not the empirical mean. We do this for the sake of clarifying our lemmas and theorems, without losing the crux of the result. Indeed, by the result of Lemma~\ref{lemma:covariance_error}, we see that the difference between what we do and estimating also the means decay as $O(1/N)$. However the error in estimating the covariance term decays as $O(1/\sqrt{N})$ (see Lemma~\ref{lemma:error_emp_covariance} and Lemma~\ref{lemma:error_emp_cross-covariance}), which means that for a reasonable number of samples, the error due to the shift in mean becomes negligible with regards to the covariance estimation. This practically means that in the case where one uses the empirical mean, all of our result would hold with slightly bigger multiplicative constants.




\section{Numerical Experiments}
\label{sec:numerics}

In this section we present a set of numerical experiments illustrating the
behaviour of the linear LMMSE estimators introduced in
Sections~\ref{sec:LLMSE_blind} and~\ref{sec:LLMSE_empirical} under uncertainty on both the data
distribution and the measurement operator.  
All simulations are performed on synthetic one-dimensional signals of
length $n$. For this, we first
provide in Section \ref{subsec:illustrative} several numerical details clarifying the statistics of the
synthetic signals and measurements generated throughout the numerical study.

As a first numerical test, we focus on the approximation error bound under spectral regularization derived in Theorem~\ref{th:error_bound_singular_values}, and verify numerically that the reconstruction error behaves in accordance with the source-condition–dependent constants.

As a second test, we perform numerical verifications of the results reported in Section~\ref{sec:variance_error}.  
In particular, we focus on the sampling-induced variance component of the LMMSE estimator and show numerically that the empirical operator $\hat{\Lv}_\lambda$ converges to its population counterpart $\Lv_\lambda$ at the predicted rate.


\subsection{Data Generation: Sinusoidal Signals with Uncertain Measurements}
\label{subsec:illustrative}

We begin by describing the synthetic data model used throughout the
numerical section.  
Each ground-truth signal $\xv=(x_1,x_2,\ldots,x_n) \in \mathbb{R}^n$ is generated as a
non-zero-mean sinusoid corrupted by additive Gaussian perturbations,
namely, for $i=1,\ldots,n$
\begin{equation} \label{}
x_i= b\sin\left(\frac{2\pi\,i}{n} \right) + z_i,
\end{equation}
where  $b>0$, and $\zv \sim \mathcal{N}(0,b
\Id_n)$ represents random variability across the signal family.  
In 
Figure~\ref{fig:data_generation1}, we report the exact and empirical signal estimated by a collection of $N=1000$ i.i.d.\ samples
$\{ \xv^j \}_{j=1}^N$. The shaded areas indicate $\theta_{\xv} \pm \hat{\sigma}_{\xv}$, where
$\hat{\sigma}_{\xv}$ denotes the square root of the empirical diagonal of
$\Cxx$ with $b=2.$ The corresponding (paired) measurements $\{ \yv^j \}_{j=1}^N$ are generated by convolving each sample $\xv^j$ with an
uncertain smoothing kernel and additive white Gaussian noise (AWGN). These kernels are generated as follows: let $\bar{\mathbf{k}} \in \mathbb{R}^n$ denote a fixed
reference kernel, chosen here as a normalized Gaussian one.  For each
realization $j=1,\ldots,N$ and each $i=1,\dots,n$, we draw an independent perturbation of the kernel mean,
modeled as
\[
\mathbf{k}^j_i = \bar{\mathbf{k}}_{i - \theta^j},
\qquad 
\theta^j \sim \mathcal{N}(0,\sigma_\theta^2),
\]
so that the effective kernel is shifted randomly along the signal
domain.  By incorporating also AWGN $\veps\sim\mathcal{N}(\mathbf{0},\beta\mathbf{I}_n)$, the resulting modelling process thus takes the form:
\[
\yv^j = \mathbf{k}^j\ast \xv^j + \veps
\]
and it can be employed to model uncertainty in the alignment or calibration of
measurement devices, depending on the parameters $\theta$ and $\sigma_n$ controlling the level of
operator uncertainty and measurement noise, respectively.

\begin{remark}[Computation of covariance with convolution matrices]
    In the setting of our experiments, the sampled random operators $\fop^j$ are
    convolution matrices.  It is therefore useful to make explicit how the
    general formulas in Sections~\ref{sec:LLMSE_blind} and
    \ref{sec:LLMSE_empirical} specialize in this structured case and how
    one can exploit this structure in numerical experiments.

    We thus consider $m=n$ and  $\fop\in\R^{n\times n}$ to be
    the matrix associated with a (vectorised) convolution kernel
    $\mathbf{k}\in\R^n$ (for instance, the vectorised form of a
    $d\times d$ kernel with $d^2=n$).  Under periodic boundary
    conditions, such matrices are block-circulant with circulant blocks
    (BCCB) and can be represented as linear combinations of ``shift''
    matrices.  More precisely, let $\{\mathbf{B}_i\}_{i=1}^n$ denote the
    family of BCCB matrices corresponding to a unit impulse located at
    position $i$, so that applying $\mathbf{B}_i$ amounts to circularly
    shifting the signal by $i$ pixels. Then, $\fop$ can
    be written as
    \[
        \fop = \sum_{i=1}^n k_i \mathbf{B}_i 
        =: T(\mathbf{k}),
    \]
    where $T:\R^n\to\R^{n\times n}$ is a linear map taking kernels as input and outputting 
    convolution matrices.

    Since $T$ is linear, it admits a matrix representation:  by stacking
    the coefficients of $\fop$ into a vector, we obtain
    \[
        \mathrm{vec}(\fop)
        \;=\;
        \mathbf{P}\,\mathbf{k},
        \qquad
        \mathbf{P}
        =
        \bigl[\mathrm{vec}(\mathbf{B}_1),\ldots,
              \mathrm{vec}(\mathbf{B}_n)\bigr]
        \in\R^{n^2\times n},
    \]
    that is, each column of $\mathbf{P}$ contains the vectorised version
    of one shift matrix.  This representation is particularly convenient
    when $\mathbf{k}$ is random as it allows to transfer knowledge on the distribution on $\mathbf{k}$ directly into the distribution of $\fop$.  Precisely, if $\mathbf{k}$ has mean
    $\theta_{\mathbf{k}}\in\R^n$ and covariance
    $\Cv_{\mathbf{k},\mathbf{k}}\in\R^{n\times n}
\in\R^{n\times n}$, then by
    linearity of expectation and covariance we obtain
    \[
        \mufop 
        = \Expect{\Av}{\Av} 
        = \Expect{\mathbf{k}}{T(\mathbf{k})} 
        = T\!\bigl(\Expect{\mathbf{k}}{\mathbf{k}}\bigr)
        = T(\theta_{\mathbf{k}}),
    \]
    so the mean operator is simply the convolution matrix associated with
    the average kernel.

    Furthermore, the covariance of $\fop$ is obtained by expanding
    \[
        \fop - \mufop
        = \sum_{i=1}^n (k_i - (\theta_{\mathbf{k}})_i)\,\mathbf{B}_i.
    \]
    Using bilinearity of the covariance and the fact that the
    $\mathbf{B}_i$ are deterministic, we obtain
    \[
        \Cv_{\fop,\fop}
        \;=\;
        \sum_{i=1}^n\sum_{j=1}^n 
        \Cv_{k_i,k_j}\,
        (\mathbf{B}_i \otimes \mathbf{B}_j),
    \]
    where each covariance entry $\Cv_{k_i,k_j}$ multiplies the
    appropriate shift-matrix pair.  Thus all second-order statistics of
    $\fop$ are fully determined by the statistics of the kernel
    $\mathbf{k}$ and the fixed matrices $\mathbf{B}_i$. From a computational perspective, this representation is extremely
    advantageous: once the shift matrices $\mathbf{B}_i$ are
    precomputed, updating $\mufop$ and $\Cv_{\fop,\fop}$ for different
    kernel distributions reduces to updating the much lower-dimensional
    kernel statistics $\theta_{\mathbf{k}}$ and
    $\Cv_{\mathbf{k},\mathbf{k}}$.
    \end{remark}

\medskip

Figures~\ref{fig:data_generation2} and \ref{fig:data_generation3}
illustrate exemplar realizations of $\yv$ along with samples of the
perturbed kernels respectively. In the following experiments, we use a Gaussian kernel $\bar{\mathbf{k}}$ of size $d=21$ with spread $\varsigma_{\bar{\mathbf{k}}}=0.5$ and kernel-shift variability $\sigma_\theta=0.4$, together with AWGN of level $\sigma_n=0.5$.
 For further illustration, we report in Figure \ref{fig:samples_sinusoidal}, four samples $\left\{ \xv^j\right\}_{j=1}^4$ and the corresponding  $\left\{ \yv^j\right\}_{j=1}^4$ generated as above to show dataset diversity.

\begin{figure}[t]
\centering
\begin{subfigure}{0.49\textwidth}
    \centering
    \includegraphics[width=\textwidth]{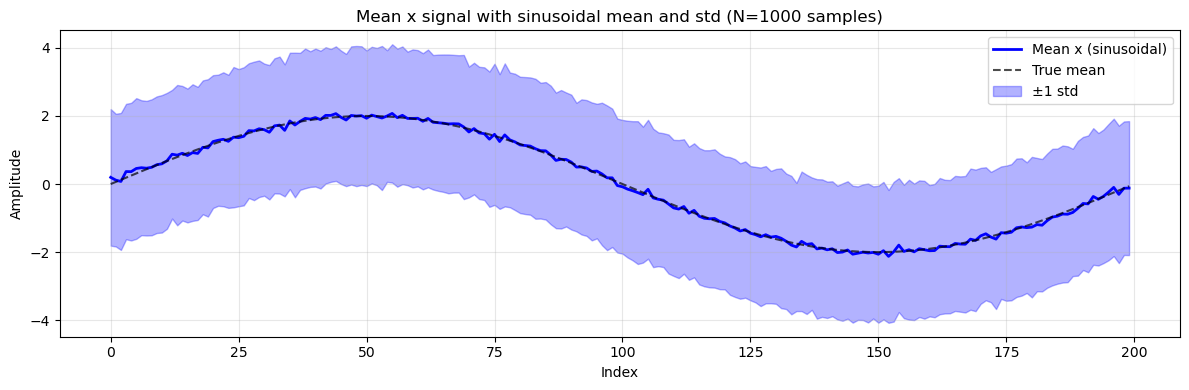}
    \caption{Reference signal statistics, $\xv$.}
    \label{fig:data_generation1}
\end{subfigure}
\begin{subfigure}{0.49\textwidth} 
    \centering
    \includegraphics[width=\textwidth]{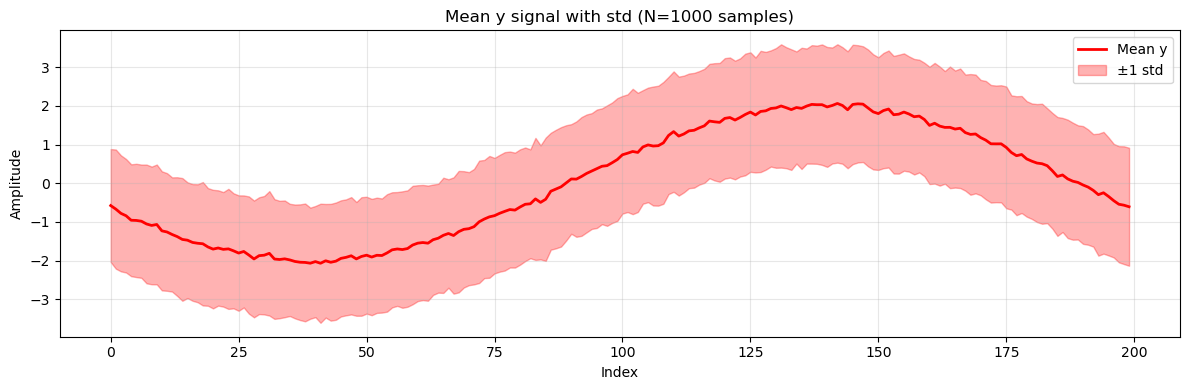}
    \caption{Observed signal statistics, $\yv$.}
    \label{fig:data_generation2}
\end{subfigure}

\begin{subfigure}{0.49\textwidth}
    \centering
    \includegraphics[width=\textwidth]{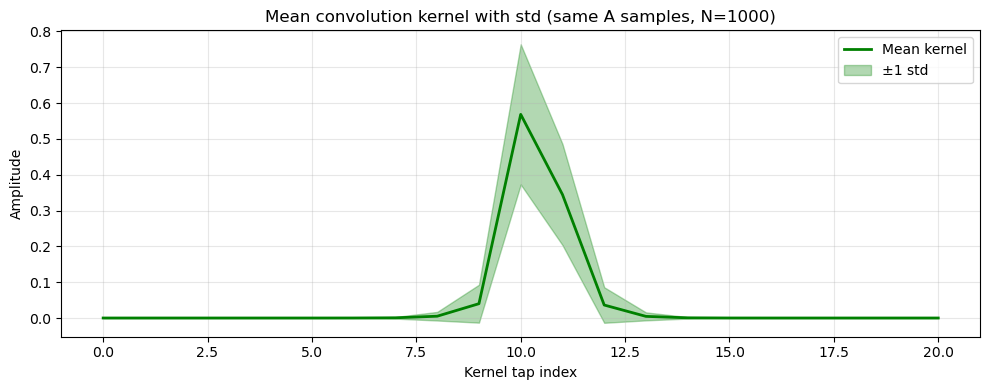}
    \caption{Convolution kernel statistics, $\mathbf{k}$.}
\label{fig:data_generation3}
\end{subfigure}
\caption{Illustrative synthetic dataset:  statistics of a sinusoidal signal $\xv$ (Figure \ref{fig:data_generation1}), observations obtained by convolution with Gaussian kernel (Figure \ref{fig:data_generation2}) and convolution kernels (Figure \ref{fig:data_generation3}) with uncertainties. For all quantities, empirical means are denoted by solid lines, together with $\pm$ one empirical standard deviation (shaded area) computed from $N=1000$ training samples.}
 \label{fig:data_generation}
\end{figure}

\begin{figure}[h!]
    \centering
    \begin{subfigure}{\textwidth}
        \centering
        \includegraphics[width=\textwidth]{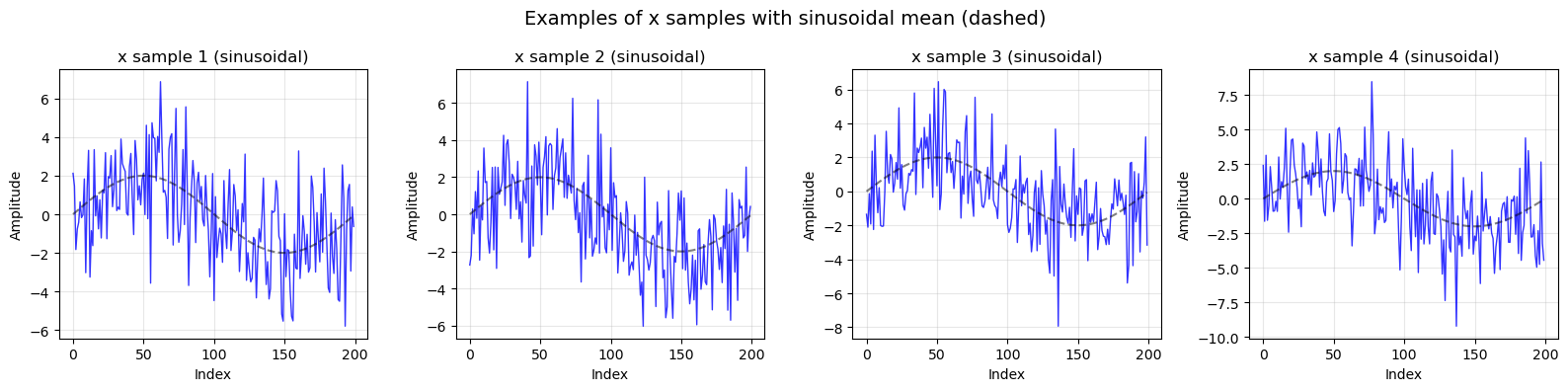}
    \end{subfigure}

    \begin{subfigure}{\textwidth}
        \centering
        \includegraphics[width=\textwidth]{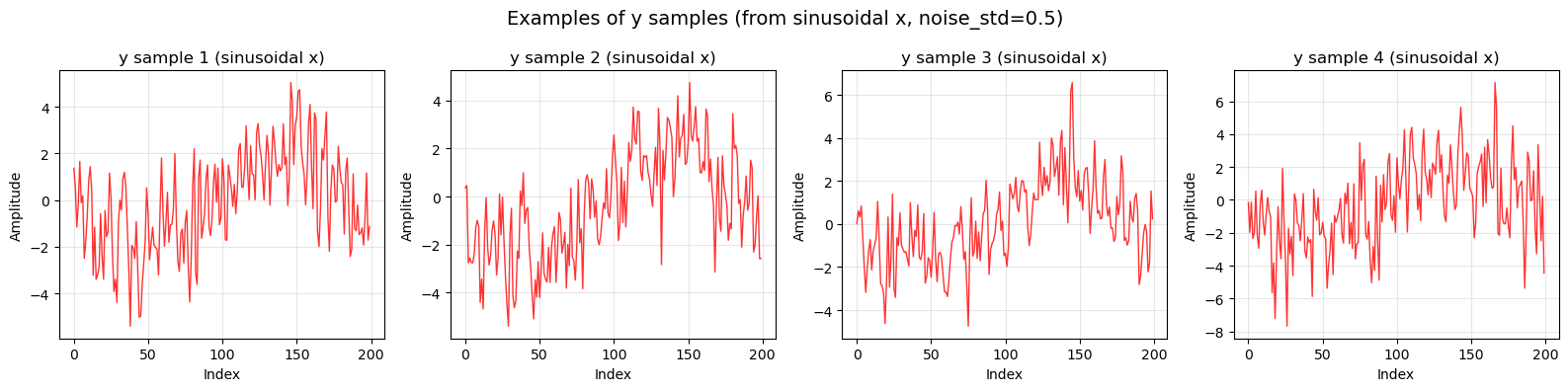}
    \end{subfigure}
    \caption{Samples of reference and observed (blurred + noisy) signals.}
    \label{fig:samples_sinusoidal}
\end{figure}

As a first sanity check, we validate the formula \eqref{eq:LMMSE-B} for the computation of the LMMSE estimator
$\hat{\xv}_{\text{LMMSE}}$ 
using the exact (i.e., non-sampled) statistical information on a test signal not included
in the training set, see Figure \ref{fig:LMMSE_signal_test}. This example highlights the qualitative behaviour of the estimator, and
illustrates the role played by the uncertainty on both $\xv$ and $\mathbf{k}$.

\begin{figure}[h!]
    \centering
    \includegraphics[width=0.8\textwidth]{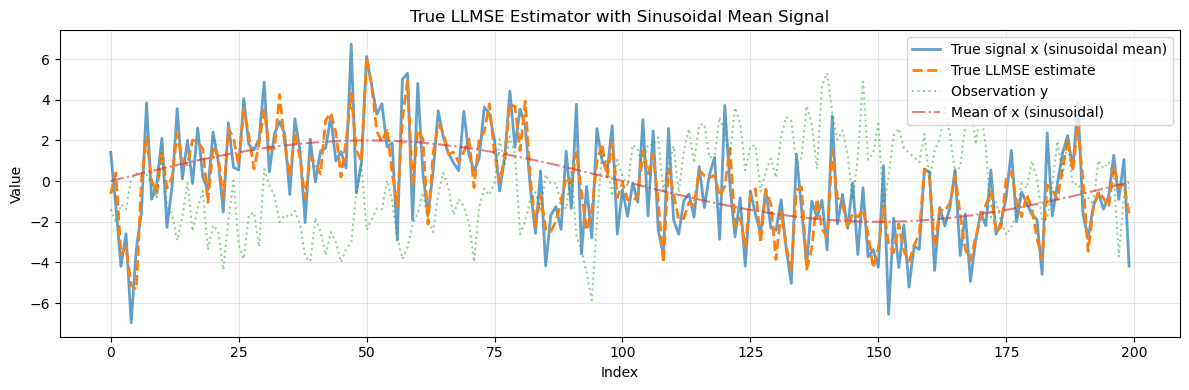}
    \caption{Reconstruction of a test signal using the theoretical LMMSE
    estimator. Ground truth (solid blue), sinusoidal mean (dotted red), measurement (dotted green), and
   theoretical-LMMSE reconstruction (dashed orange).}
    \label{fig:LMMSE_signal_test}
\end{figure}

\subsection{Numerical Verification of the Reconstruction Error}
\label{subsec:rec_error_approx_numerics}

We now turn to the numerical verification of the approximation error
bound established in Theorem~\ref{th:error_bound_singular_values} for the
spectrally regularized LMMSE estimator.  
In particular, we show here that this
bound is numerically meaningful and that its dependence on $\alpha$ and
on the operator variability encoded in the coefficient of variation matches the predicted behaviour. All experiments in this subsection reuse the data model introduced in
Section~\ref{subsec:illustrative}. The kernel parameters are  chosen as in
Section~\ref{subsec:illustrative}; they will be varied explicitly only
when studying the effect of kernel variability. To enforce the source condition of
Section~\ref{sec:LLMSE_blind}, we construct the prior
covariance of the signal $\xv$ as
\[
    \Cxx \;=\; \Expect{\fop}{\pa{\fop\tp\fop}}^\alpha,
\]
where the expectation is taken with respect to the distribution of the
random convolution operator $\fop$.  Numerically, we proceed as follows.
First, we generate $N$ i.i.d.\ realizations
$\{\Av^j\}_{j=1}^{N}$ of the convolution matrix, obtained
from independent draws $\{\mathbf{k}^j\}_{j=1}^{N}$.  We then form the empirical average
\[
\widehat{\Cv}_{\Av}
    \;\coloneqq\;
    \frac{1}{N} \sum_{j=1}^{N} (\Av^j)^\top (\Av^j),
\]
and define
\[
\Cxx:=\widehat{\Cv}_{\Av}^{\,\alpha},
\]
where the power is taken by raising the eigenvalues (or,
equivalently, the diagonal entries in the Fourier domain) of
$\widehat{\Cv}_{\Av}$ to the power $\alpha$.  Finally, we generate the
training signals from the Gaussian prior
\[
    \xv^j \sim \mathcal{N}(\theta_{\xv}, \Cxx),
    \qquad
    j = 1,\dots,N,
\]
where $\theta_{\xv}$ is the sinusoidal signal introduced in
Section~\ref{subsec:illustrative}.  
The corresponding
measurements are then generated as
\[
    \yv^j = \fop^j \xv^j + \veps^j,
    \qquad
    \veps^j \sim \mathcal{N}(\mathbf{0},\sigma_n^2 \Id_n),
\]
using independent draws $\fop^j$ of the convolution operator with the
same kernel statistics as above.

For a fixed $\alpha$ and a given configuration of kernel variability and
noise level, we compute the empirical quantities required
by the LMMSE formula and apply it on a test observation $\yv$, obtaining the regularized estimator
$\hat{\xv}_\alpha = \Lv_{\alpha}^\lambda \yv$ (see
Sections~\ref{sec:LLMSE_blind} and~\ref{sec:LLMSE_empirical}).  On the
test set, we thus evaluate the empirical reconstruction error
\begin{equation} \label{eq:LHS}
    \frac{1}{K} \sum_{k=1}^K
    \bigl\| \xv^k - \hat{\xv}^k_{\alpha} \bigr\|_2^2,
\end{equation}
and compare it with the right-hand side $\mathrm{RHS}(\alpha)$ provided
by the approximation-error theorem in
Section~\ref{sec:bias_error}.  The theoretical bound can be
decomposed into two main contributions,
\[
    \mathrm{RHS}(\alpha)
    \;=\;
    \mathrm{term}_1(\alpha) + \mathrm{term}_2(\alpha,\mathrm{CV}^2),
\]
where $\mathrm{term}_1(\alpha)$ depends on the source condition and
regularization parameters through explicit constants $C_1(\alpha)$ and
$C_2(\alpha)$, while $\mathrm{term}_2(\alpha,\mathrm{CV}^2)$ captures
the effect of the kernel variability via the coefficient of variation
squared $\mathrm{CV}^2$ of the width parameter $\sigma$ as discussed after theorem~\ref{th:main_aprrox_and_sampling}

\paragraph{Dependence on the source condition parameter $\alpha$.}
We first fix the kernel statistics to an intermediate regime (medium
noise and medium operator variability), as in
Section~\ref{subsec:sampling_bounds_numerics}, and vary the source
condition exponent in the set
\[
    \alpha \in \{0.5,\, 1.0,\, 1.3,\, 1.5,\, 1.8,\, 2.0\}.
\]
For each value of $\alpha$ we generate $N=1000$ training samples and
evaluate both \eqref{eq:LHS} and $\mathrm{RHS}(\alpha)$ on $K=20$
test signals.  Figure~\ref{fig:rec_error_alpha} reports both quantities
 in log scale as functions of $\alpha$. 
In all tested configurations, the bound \eqref{eq:approx_bound} is satisfied with a
moderate margin, and the dependence of the bound on $\alpha$ is
consistent with the theoretical prediction: the constants
$C_1(\alpha)$, $C_2(\alpha)$ and $C_3(\alpha)$ exhibit the expected
monotonic trends, and the tightness of the bound varies smoothly with
the smoothness encoded by the source condition.

\begin{figure}[h!]
        \centering
    \includegraphics[width=0.5\textwidth]{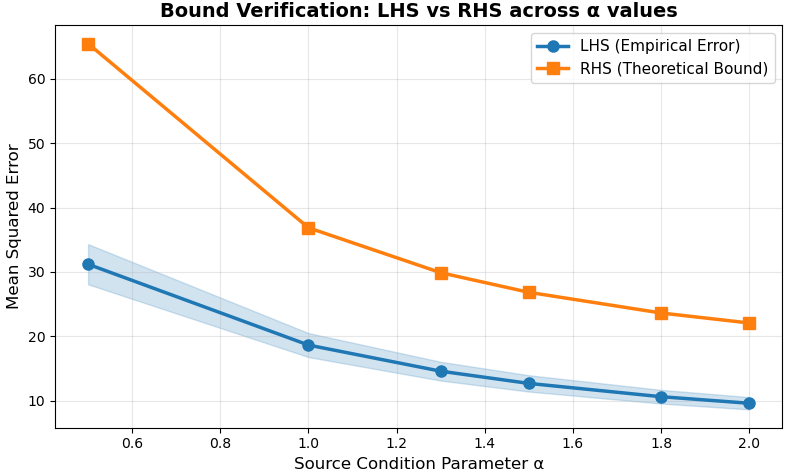}
        \caption{Numerical verification of the error bound \eqref{eq:approx_bound} as a
    function of the source condition parameter $\alpha$.}
    \label{fig:rec_error_alpha}
\end{figure}

\paragraph{Dependence on kernel variability and $\mathrm{CV}^2$.}
We next fix the source condition parameter to a representative value,
say $\alpha = 1.2$, and study the impact of the kernel variability on
the bound.  To this end, we vary the standard deviation
$\sigma_{\mathrm{std}}$ of the kernel width around its mean
$\mu_\sigma$, which in turn changes the coefficient of variation
$\mathrm{CV}(\sigma)$ and its square $\mathrm{CV}^2(\sigma)$ which we noted is directly linked to our bound in Theorem~\eqref{th:main_aprrox_and_sampling}.  For each
choice of $\sigma_{\mathrm{std}}$ we generate $N=1000$ training samples,
compute the corresponding LMMSE estimator, and evaluate the empirical
error \eqref{eq:LHS} and the theoretical right-hand side.  The results are summarized
in Figure~\ref{fig:rec_error_cv2}.

Consistently with the theoretical analysis, increasing the variability
of the kernels (i.e., larger $\sigma_{\mathrm{std}}$ and hence larger
$\mathrm{CV}^2(\sigma)$) leads to an increase of the second term in the
bound and, consequently, to a larger $\mathrm{RHS}$. 


\begin{figure}[h!]
    \centering
\includegraphics[width=\textwidth]{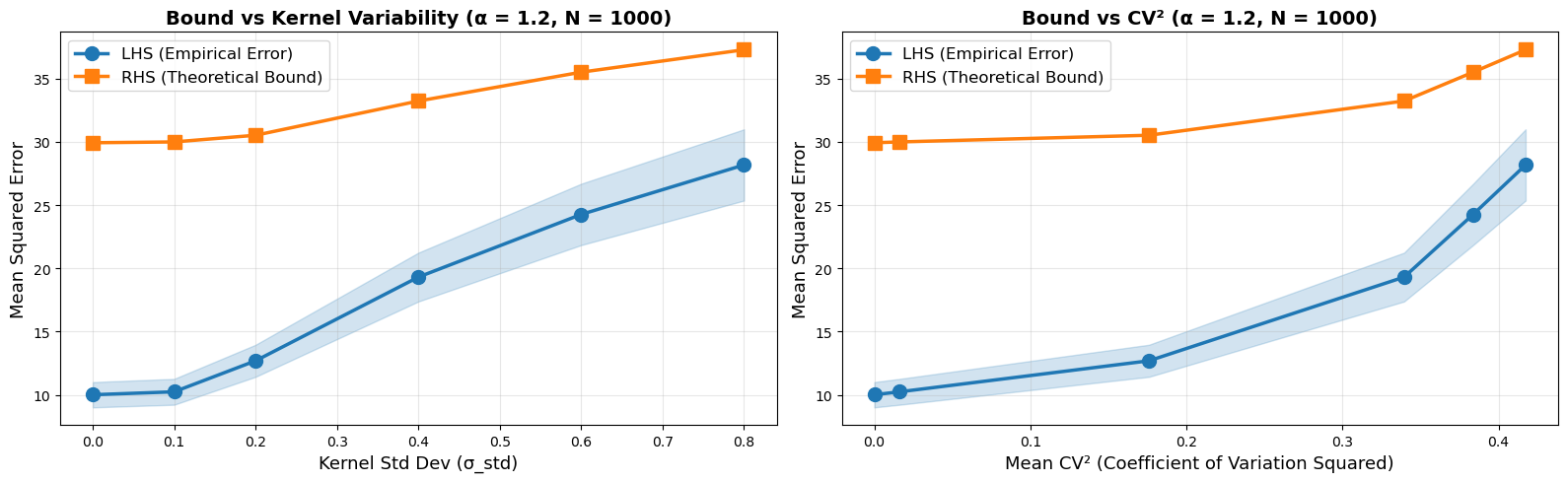}
   \caption{Numerical verification of the approximation error bound \eqref{eq:approx_bound} as a
    function of kernel variability, parameterized via the standard deviation $\sigma_{\mathrm{std}}$ (or equivalently $\mathrm{CV}^2$) of the kernel width.}

    \label{fig:rec_error_cv2}
\end{figure}

\subsection{Numerical Verification of the Sampling Bounds}
\label{subsec:sampling_bounds_numerics}

We now turn to a second numerical experiment, which examines the
convergence of the empirical LMMSE estimator $\hat{\Lv}_\lambda$
towards its population counterpart $\Lv_\lambda$ as the number of
training samples $N$ increases.  
The theoretical results of Section~\ref{sec:variance_error}, and in particular Theorem \ref{th:main_sampling_bound} and Corollary \ref{corollary:sampling_bound} predict
that, under the assumptions of our model, the reconstruction error
decays at rate $O(N^{-1})$ with $N$ denoting the sample size.

To validate this behaviour, we generate $K=20$ independent sinusoidal
signals as described in Section~\ref{subsec:illustrative}, each with
randomly perturbed amplitudes, phases, and baselines.  
For each signal, we consider three experimental regimes of varying difficulty:
\begin{enumerate}[label=(R\arabic*)]
    \item  medium AWGN level (medium $\sigma_n$) and average operator uncertainty (medium $\sigma_\theta$);
    \item high AWGN level (large $\sigma_n$)  and average operator uncertainty (medium $\sigma_\theta$);
    \item average AWGN level (medium $\sigma_n$) and high operator uncertainty (large $\sigma_\theta$).
\end{enumerate}
For each regime, we draw $N\in\left\{500, 1000, 1500, 2000, 2500, 3000, 3500, 4000\right\}$ training samples and compute the empirical LMMSE estimators $\hat{\Lv}^\lambda$
and
evaluate the sample reconstruction error
\[
    \text{err}_N 
    \;\coloneqq\;
    \frac{1}{N}\sum_{n=1}^N \big\| \hat{\Lv}^\lambda \yv^n - \Lv^\lambda \yv^n \big\|_2^2
\]
on the test signals, then averaging over $k$ experiments to study variability across experiments.

Figure~\ref{fig:convergence-LMMSE} reports the mean error $\text{err}_N $ as a function of $N$ for all
three regimes.  
All curves exhibit a decay consistent with the predicted
$1/N$ convergence rate of Theorem~\ref{th:main_sampling_bound} and Corollary~\ref{corollary:sampling_bound}.  
This numerical evidence confirms the robustness of the empirical LMMSE
approach under realistic levels of noise and operator variability. 

\begin{figure}[h!]
    \centering
    \includegraphics[height=6cm]{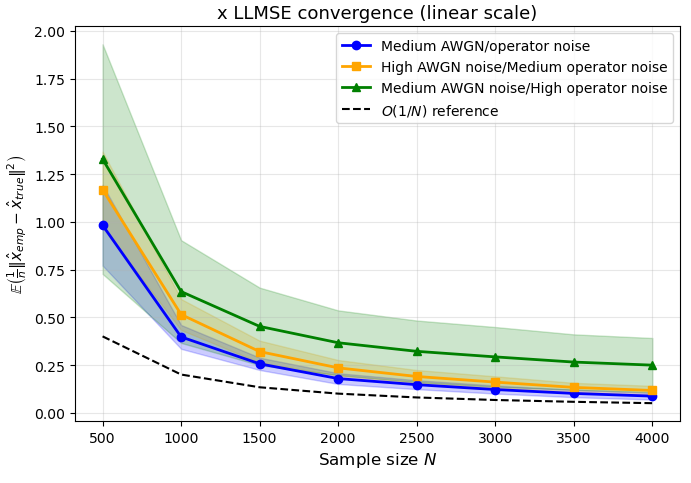}
    \caption{Convergence of the empirical LMMSE estimator: average
    reconstruction error $\text{err}_N$ vs.\ sample size $N$ for the three
    noise/uncertainty regimes described in the text. The shaded regions indicate empirical variability across the $K=20$ test signals.
 The observed slope
    is consistent with the predicted $1/N$ decay.}
    \label{fig:convergence-LMMSE}
\end{figure}

\section{Proofs}

\subsection{Proof of equivalence between LMMSE estimators and Tikhonov-regularized problems}
We prove the two lemmas that link the LMMSE estimators of $\xv$ and $\fop$ with a specific Tikhonov-regularized problems:

\subsubsection*{Proof of Lemma~\ref{th:LMMSE-TikhG}}\label{proof:lmmse-tikhG}
We consider~\eqref{eq:LMMSE-B} noticing we can formulate it as
\begin{align*}
\hat{\xv} = \Cxx\Theta_\fop\tp\pa{\Theta_\fop \Cxx \Theta_\fop\tp + \Cp}\inv \pa{\yv - \theta_\yv} + \mux.
\end{align*}
We use the Woodbury matrix identity to expand the inverse term, obtaining:
\[
\hat{\xv} = \Cxx\pa{\Theta_\fop\tp\Cp\inv - \Theta_\fop\tp\Cp\inv\Theta_\fop\pa{\Cxx\inv + \Theta_\fop\tp\Cp\inv\Theta_\fop}\inv\Theta_\fop\tp\Cp\inv}\pa{\yv - \theta_\yv}+ \mux.
\]
We can factor out $\Theta_\fop\tp\Cp\inv$ to the right side of the parenthesis:
\[
\hat{\xv} = \Cxx\pa{\Id - \Theta_\fop\tp\Cp\inv\Theta_\fop\pa{\Cxx\inv + \Theta_\fop\tp\Cp\inv\Theta_\fop}\inv}\Theta_\fop\tp\Cp\inv\pa{\yv - \theta_\yv}+ \mux.
\]
Using the identity $\Id - \mathbf{B}(\mathbf{A}+\mathbf{B})\inv = \mathbf{A}(\mathbf{A}+\mathbf{B})\inv$ with $\mathbf{A}=\Cxx\inv$ and $\mathbf{B}=\Theta_\fop\tp\Cp\inv\Theta_\fop$, the term in the brackets simplifies to $\Cxx\inv(\Cxx\inv + \Theta_\fop\tp\Cp\inv\Theta_\fop)\inv$. By using this, we can obtain:
\[
\hat{\xv} = \pa{\Cxx\inv + \Theta_\fop\tp\Cp^{-1}\Theta_\fop}\inv \Theta_\fop\tp\Cp\inv\pa{\yv - \theta_\yv}+ \mux
\]
which is the solution of~\eqref{eq:Tikh-sig}, proving our claim.

%

\subsubsection*{Proof of Lemma~\ref{th:LMMSE_fop_Tikh}}\label{proof:lmmse-tikh-op}
Let us define for notational purposes $\Gv = \pa{\Id_m \otimes \xv\tp}$. Let us recall that the covariance of $\yv$ is given by~\eqref{eq:Cyy_blind}:
\begin{align*}
\Cyy &= \theta_\fop\Cxx\theta_\fop\tp + \Theta_\Gv\Caa\Theta_\Gv\tp + \Dv + \Ceps.
\end{align*}
This means that we can start from the solution of the Tikhonov-regularized problem of Lemma~\ref{th:LMMSE_fop_Tikh} which is
\begin{align*}
\pa{\Theta_\Gv\pa{\theta_\fop\Cxx\theta_\fop\tp+\Dv+\Ceps}\inv\Theta_\Gv\tp + \Caa\inv}\Theta_\Gv\tp \pa{\theta_\fop\Cxx\theta_\fop\tp+\Dv+\Ceps}\inv\yv,
\end{align*}
and from there use Woodbury identity to simplify it and get back to the LMMSE estimator formulation. That is, we obatin
\begin{align*}
\pa{\Theta_\Gv\pa{\theta_\fop\Cxx\theta_\fop\tp+\Dv+\Ceps}\inv\Theta_\Gv\tp + \Caa\inv}\Theta_\Gv\tp \pa{\theta_\fop\Cxx\theta_\fop\tp+\Dv+\Ceps}\inv = \Cay\Cyy\inv
\end{align*}
where the right-hand side is the LMMSE to recover $\fop$ and the left hand-side is the solution of the Tikhonov problem~\eqref{eq:Tikh-sig}, which concludes the proof.

\subsection{Proofs for Theorem~\ref{th:error_bound_singular_values}}

We start by proving some lemmas that will be necessary for the proof Theorem~\ref{th:error_bound_singular_values}:

\begin{lemma}[Operator norm of $\Lv_\lambda$]\label{lemma:lip_L_lambda}
Assume~\ref{ass:fop_share_sing_vec} and assume that $\Cxx$ follows the source condition~\ref{def:source_cond_random_sv} with $\alpha\in\R_+^*$. Then, the norm of $\Lv^\lambda_\xv$ is bounded by:
\begin{align*}
\norm{\Lv^\lambda_\xv} \leq C\pa{\beta+\lambda}^{-\frac{1}{2(\alpha+1)}}
\end{align*}
with $C\in]0,1[$ depending on $\alpha$.
\end{lemma}
\begin{proof}
We start by using the fact that the random operator $\fop$ share fixed singular vectors with $\Cxx$ due to the source condition~\eqref{def:source_cond_random_sv}. To see this, let us define $\Expect{}{\fop} = \Uv\Dv_{\theta_\varsigma}\Vv$ with $\Dv_{\theta_\varsigma} \in \R^{m\times n}$ the diagonal matrix of expected singular values. From there we observe:
\begin{align*}
\Cxx = \Expect{}{\fop\tp\fop}^\alpha = \Expect{}{\Vv\tp\Dv\tp\Uv\tp\Uv\Dv\Vv}^\alpha = \Vv\tp\Expect{}{\Dv\tp\Dv}^\alpha\Vv = \Vv\tp\Dv_{\theta_{\varsigma^2}}^\alpha\Vv 
\end{align*}
where we define $\Dv_{\theta_{\varsigma^2}} = \Expect{}{\Dv\tp\Dv}$. Let us build on this and observe that:
\begin{align*}
\Expect{}{\fop\Cxx\fop\tp} = \Uv\Expect{}{\Dv\Dv_{\theta_{\varsigma^2}}^\alpha\Dv\tp}\Uv\tp =  \Uv\Dv_\Sigma\Uv\tp
\end{align*}
where we define $\Dv_\Sigma = \Expect{}{\Dv\Dv_{\theta_{\varsigma^2}}^\alpha\Dv\tp}$.
 We get:
\begin{align*}
\norm{\Lv^\lambda_\xv} &= \norm{\Cxx\Expect{}{\fop\tp}\pa{\Expect{}{\fop\Cxx\fop\tp} + \Ceps + \lambda\Id}\inv}\\
&= \norm{\Vv\tp\Dv_{\theta_{\varsigma^2}}^\alpha\Dv_{\theta_\varsigma}\tp\Uv\tp \pa{\Uv\Dv_\Sigma\Uv\tp + \pa{\beta+\lambda}\Id}\inv}\\
&= \norm{\Vv\tp \Dv_{\theta_{\varsigma^2}}^\alpha\Dv_{\theta_\varsigma}\tp \pa{\Dv_\Sigma + \pa{\beta+\lambda}\Id}\inv\Uv\tp }\\
&=\max_{i \in \{1\dots m\}} \pa{\Dv_{\theta_{\varsigma^2}}^\alpha\Dv_{\theta_\varsigma}\tp \pa{\Dv_\Sigma  + \pa{\beta+\lambda}\Id}\inv}_{i,i}.
\end{align*}
Furthermore, by definition, we have 
\[
\pa{\Dv_{\theta_{\varsigma^2}}^\alpha\Dv_{\theta_\varsigma}\tp }_{i,i} \leq \pa{\Dv_{\theta_{\varsigma^2}}^{\alpha+\frac{1}{2}}}_{i,i} \quad \text{ and } \quad \pa{\Dv_\Sigma}_{i,i} = \pa{\Dv_{\theta_{\varsigma^2}}^{\alpha+1}}_{i,i}\ \quad \text{ for } i\in\{1,\dots,m\}.
\]
Let us consider the following function:
\begin{align}\label{eq:trick_noise}
f(x) = \frac{x^{\alpha+1/2}}{x^{\alpha+1}+\beta+\lambda}.
\end{align}
where by construction we have $\norm{\Lv^\lambda_\xv} \leq \max_{x\in\R^+} f(x)$. In order to find the maximum of $f$, we look for $x^*$ such that $f'(x^*) = 0$. By standard calculus we have that
\begin{align*}
f'(x) = \frac{x^{\alpha}}{\pa{x^{\alpha+1}+\beta}^2} \pa{\pa{\alpha+1/2}\beta - x^{\alpha+1}},
\end{align*}  
which means that to maximize $f$ we need to take $x^* = \pa{\beta\pa{\alpha+1/2}}^{\frac{1}{\alpha+1}}$. Inserting this value into $f$ we get that:
\begin{align*}
\norm{\Lv^\lambda_\xv} \leq f(x^*) = \frac{\pa{\alpha+1/2}^{\frac{\alpha+1/2}{\alpha+1}}}{\alpha+1}\pa{\beta+\lambda}^{-\frac{1/2}{\alpha+1}}
\end{align*}
\end{proof}

\begin{lemma}\label{lemma:approximation_error}
Under the same assumptions as in Lemma~\ref{lemma:lip_L_lambda}, the $m$ first singular values $s_i$ for $i\in\{1,\dots,m\}$ of $\Expect{}{\pa{\Lv^\lambda_\xv\fop - \Id}\tp\pa{\Lv^\lambda_\xv\fop - \Id}}$ are bounded as follows:
\begin{align*}
\sum_{i=1}^m s_i \leq \sum_{i=1}^m \frac{\Vari{\varsigma_i}}{\Expect{}{\varsigma_i^2}} + \frac{(\beta+\lambda)^2}{\pa{\Expect{}{\varsigma_i^2}^{\alpha+1}+\beta+\lambda}^2}.
\end{align*}
\end{lemma}
\begin{proof}
Let us start by observing that, by using the notation defined in Lemma~\ref{lemma:lip_L_lambda}, we have:
\[
\Lv^\lambda_\xv\fop - \Id = \Vv\tp\pa{\Dv_{\Expect{}{\varsigma_i^2}^{\alpha}} \Dv_{\theta_\varsigma} \pa{\Dv_\Sigma + \pa{\beta+\lambda}\Id}\inv\Dv - \Id}\Vv.
\]

Thus, for the $i$-th singular value $s_i$, $i=1,\dots,m$  (as the others $n-m$ ones will be $-1$ since the diagonal matrices $\Dv,\Dv_{\Expect{}{\varsigma_i^2}^{\alpha}}, \Dv_{\theta_\varsigma}$ and $\Dv_\Sigma$ have at most $m$ entries), we observe by the definition of each of the diagonal matrices that it can be written as:
\begin{align*}
s_i &= \Expect{}{\pa{\frac{\Expect{}{\varsigma_i^2}^\alpha\Expect{}{\varsigma_i}\varsigma_i}{\Expect{}{\varsigma_i^2}^{\alpha+1}+\beta+\lambda} - 1}^2}\\
&= \frac{\Expect{}{\varsigma_i^2}^{2\alpha+1} \Expect{}{\varsigma_i}^2}{\pa{\Expect{}{\varsigma_i^2}^{\alpha+1}+\beta+\lambda}^2} - 2\frac{\Expect{}{\varsigma_i^2}^{\alpha} \Expect{}{\varsigma_i}^2}{\Expect{}{\varsigma_i^2}^{\alpha+1}+\beta+\lambda} + 1
\end{align*}
We can simplify this expression to get 
\begin{align*}
s_i = 1 - \frac{\Expect{}{\varsigma_i^2}^\alpha\Expect{}{\varsigma_i}^2\pa{\Expect{}{\varsigma_i^2}^{\alpha+1} + 2(\beta + \lambda)}}{\pa{\Expect{}{\varsigma_i^2}^{\alpha+1}+\beta+\lambda}^2}.
\end{align*} 
Using  assumption~\ref{ass:fop_share_sing_vec} we have that $\Expect{}{\varsigma_i^2} > 0$. We can thus write equivalently
\begin{align*}
s_i &= 1 - \frac{\Expect{}{\varsigma_i^2}^{\alpha+1} \pa{\Expect{}{\varsigma_i^2}^{\alpha+1} + 2(\beta + \lambda)} \frac{\Expect{}{\varsigma_i}^2}{\Expect{}{\varsigma_i^2}}}{\pa{\Expect{}{\varsigma_i^2}^{\alpha+1}+\beta+\lambda}^2}\\
&= \frac{\pa{\Expect{}{\varsigma_i^2}^{2\alpha+2} + 2(\beta+\lambda)\Expect{}{\varsigma_i^2}^{\alpha+1}} \pa{1 - \frac{\Expect{}{\varsigma_i}^2}{\Expect{}{\varsigma_i^2}} } + (\beta + \lambda)^2}{\pa{\Expect{}{\varsigma_i^2}^{\alpha+1}+\beta+\lambda}^2}.
\end{align*}
From here, we use that $\Expect{}{\varsigma_i^2}^{2\alpha+2} + 2(\beta+\lambda)\Expect{}{\varsigma_i^2}^{\alpha+1} \leq \pa{\Expect{}{\varsigma_i^2}^{\alpha+1}+\beta+\lambda}^2$, alongside $\pa{1 - \frac{\Expect{}{\varsigma_i}^2}{\Expect{}{\varsigma_i^2}} } = \frac{\Vari{\varsigma_i}}{\Expect{}{\varsigma_i^2}}$ which gives the desired result:
\begin{align*}
\sum_{i=1}^m s_i \leq \sum_{i=1}^m \frac{\Vari{\varsigma_i}}{\Expect{}{\varsigma_i^2}} + \frac{(\beta+\lambda)^2}{\pa{\Expect{}{\varsigma_i^2}^{\alpha+1}+\beta+\lambda}^2}.
\end{align*}

\end{proof}

\subsubsection{Proof of Theorem~\ref{th:error_bound_singular_values}}\label{subsec:proof_approx}

Our proof follows the same basic structure as for the deterministic case, that is: we first separate the error in two terms, one for the noise error and one for the approximation error, and then we use the source condition to show that these errors are bounded.
 
\textbf{Step 1: Separate the noise and the approximation error.} We start by seeing that thanks to the independence between $\xv$ and $\veps$, we can split the error into an approximation part and a noise part. We have that:

\begin{align*}
\Expect{}{\norm{\Lv^\lambda_\xv\yv - \xvc + \bv}^2} &= \Expect{}{\norm{\Lv^\lambda_\xv\veps + \pa{\Lv^\lambda_\xv\fop - \Id}\xvc + \bv}^2}\nonumber \\
&=\Expect{}{\norm{\Lv^\lambda_\xv\veps}^2} + \Expect{}{\norm{\pa{\Lv^\lambda_\xv\fop - \Id}\xvc + \bv}^2} 
\end{align*}
where use that $\veps$ is centered to cancel the other cross terms depending on it. Now we can recall from the~\eqref{eq:LMMSE} formula that the bias term is chosen as $\bv = \Expect{}{\xv} - \Lv^\lambda_\xv\Expect{}{\yv}=\pa{\Id - \Lv^\lambda_\xv\Theta_\fop}\theta_\xv$ which gives
\begin{align*}
\Expect{}{\norm{\Lv^\lambda_\xv\yv - \xvc + \bv}^2} &= \Expect{}{\norm{\Lv^\lambda_\xv\veps}^2} + \Expect{}{\norm{\pa{\Lv^\lambda_\xv\fop - \Id}\xvc}^2} + \norm{\bv}^2 + \Expect{}{\bv\tp\pa{\Lv^\lambda_\xv\fop - \Id}\xvc} \\
&\qquad + \Expect{}{\xvc\tp\pa{\Lv^\lambda_\xv\fop - \Id}\tp\bv}\\
&= \Expect{}{\norm{\Lv^\lambda_\xv\veps}^2} + \Expect{}{\norm{\pa{\Lv^\lambda_\xv\fop - \Id}\xv}} - \norm{\bv}^2
\end{align*}


\textbf{Step 2: Bound the noise error.} We can now focus on the noise term to see that
\begin{align}\label{eq:noise_decomp_random_sv}
\Expect{}{\norm{\Lv^\lambda_\xv\veps}^2} &\leq \norm{\Lv^\lambda_\xv}^2 \Expect{}{\norm{\veps}^2} \leq \norm{\Lv^\lambda_\xv}^2 \Expect{}{\veps\tp\veps} = \norm{\Lv^\lambda_\xv}^2\Tr\pa{\Ceps} = m\beta\norm{\Lv^\lambda_\xv}^2.
\end{align}

By Lemma~\ref{lemma:lip_L_lambda}, we have that
\begin{align*}
\norm{\Lv^\lambda_\xv} \leq C_1\pa{\beta+\lambda}^{-\frac{1/2}{\alpha+1}}
\end{align*}
Now combining with~\eqref{eq:noise_decomp_random_sv}, we have that
\begin{align*}
\Expect{}{\norm{\Lv^\lambda_\xv\veps}^2} \leq m \widetilde{C}_1 \beta\pa{\beta+\lambda}^{-\frac{1}{\alpha+1}} \leq m \widetilde{C}_1 \pa{\beta+\lambda}^{\frac{\alpha}{\alpha+1}}
\end{align*}
where $\widetilde{C}_1 \in ]0,1[$.
 
\textbf{Step 3: Bound the approximation error.} Let us look at the approximation error term. We start by observing by the so-called trace trick we have:
\begin{align*}
\Expect{}{\norm{\pa{\Lv^\lambda_\xv\fop - \Id}\xvc}^2} &= \Expect{}{\xvc\tp\pa{\Lv^\lambda_\xv\fop - \Id}\tp\pa{\Lv^\lambda_\xv\fop - \Id}\xvc} \nonumber \\
&= \Expect{}{\Tr\pa{\xvc\tp\pa{\Lv^\lambda_\xv\fop - \Id}\tp\pa{\Lv^\lambda_\xv\fop - \Id}\xvc}}.\nonumber\\
\end{align*}
Furthermore, by the linearity of expectation and the cyclical property of the trace we get:
\begin{align*}
\Expect{}{\norm{\pa{\Lv^\lambda_\xv\fop - \Id}\xvc}^2} &= \Tr\pa{\Expect{\xvc}{\xvc\xvc\tp}\Expect{\fop}{\pa{\Lv^\lambda_\xv\fop - \Id}\tp\pa{\Lv^\lambda_\xv\fop - \Id}}}\nonumber \\
\end{align*}

Now we can use that $\Expect{\xvc}{\xvc\xvc\tp} = \Cxx + \theta_\xv\theta_\xv\tp$ and see by trace properties that:
\begin{align*}
\Tr\pa{\theta_\xv\theta_\xv\tp\Expect{\fop}{\pa{\Lv^\lambda_\xv\fop - \Id}\tp\pa{\Lv^\lambda_\xv\fop - \Id}}} &= \Expect{}{\Tr\pa{\theta_\xv\tp\pa{\Lv^\lambda_\xv\fop - \Id}\tp\pa{\Lv^\lambda_\xv\fop - \Id}\theta_\xv}}\\
&= \Expect{}{\norm{\bv}^2},
\end{align*}
which means:
\begin{align}\label{eq:appprox_decomp_random_sv}
\Expect{}{\norm{\pa{\Lv^\lambda_\xv\fop - \Id}\xvc}^2} - \Expect{}{\norm{\bv}^2} = \Tr\pa{\Cxx\Expect{}{\pa{\Lv^\lambda_\xv\fop - \Id}\tp\pa{\Lv^\lambda_\xv\fop - \Id}}}.
\end{align}
As only the first $m$ diagonal entries of $\Dv_{\Expect{}{\varsigma_i^2}^\alpha}$ are non-zeros, we can use Lemma~\ref{lemma:approximation_error}, from which we get:
\begin{align*}
\Expect{}{\norm{\pa{\Lv^\lambda_\xv\fop - \Id}\xvc}^2} - \Expect{}{\norm{\bv}^2} &\leq  \sum_{i=1}^m \Expect{}{\varsigma_i^2}^\alpha\frac{\Vari{\varsigma_i}}{\Expect{}{\varsigma_i^2}} + \frac{\Expect{}{\varsigma_i^2}^\alpha\beta_\lambda^2}{\pa{\Expect{}{\varsigma_i^2}^{\alpha+1}+\beta_\lambda}^2}
\end{align*}

%

The last step consists in bounding the second term in the sum. For this, we can do the same analysis as used in~\eqref{eq:trick_noise} find an upper bound on this last expression. We define the function
\begin{align*}
f(z) = \frac{\beta_\lambda^2 z^{\alpha}}{\pa{z^{\alpha+1} + \beta_\lambda}^2}
\end{align*}
and obtain by standard calculus that the derivative of $f$ is given
\begin{align*}
f'(z) = \beta_\lambda^2\frac{z^{\alpha-1}\pa{\beta\alpha - \pa{\alpha+2}z^{\alpha+1}}}{\pa{z^{\alpha+1} + \beta_\lambda}^3},
\end{align*}
which vanishes for $z^* = \pa{\frac{\alpha\beta_\lambda}{\alpha+2}}^{\frac{1}{\alpha+1}}$. plugging into $f$ we get
\begin{align*}
f(z^*)  = \pa{\frac{\alpha}{\alpha+2}}^{\frac{\alpha}{\alpha+1}} \pa{1+\frac{\alpha}{\alpha+2}}^{-2}\beta^{\frac{\alpha}{\alpha+1}} \leq C_2 \beta^{\frac{\alpha}{\alpha+1}},
\end{align*}
where $C_2\in ]0,1[$.
This means that 

\begin{align*}
\Expect{}{\norm{\pa{\Lv^\lambda_\xv\fop - \Id}\xvc}^2} - \Expect{}{\norm{\bv}^2} &< \sum_{i=1}^m C_2\beta^{\frac{\alpha}{\alpha+1}} + \Expect{}{\varsigma_i^2}^{\alpha}\frac{\Vari{\varsigma_i}}{\Expect{}{\varsigma_i^2}}.
\end{align*}

Let us finally put everything together to obtain
\begin{align*}
\Expect{}{\norm{\Lv^\lambda_\xv\yv - \xvc + \bv}^2} \leq m\pa{C_1 + C_2} \beta_\lambda^{\frac{\alpha}{\alpha+1}} +  \sum_i^m \Expect{}{\varsigma_i^2}^{\alpha}\frac{\Vari{\varsigma_i}}{\Expect{}{\varsigma_i^2}},
 \end{align*} 
 which gives the claim.

\subsection{Proofs of Theorem~\ref{th:main_sampling_bound}}

We start by giving some useful lemmas that will be of use for the proof of the theorem. We start with a lemma that bounds the Lipschitz constant of operators of LMMSE-type.
\begin{lemma}[MSE of LMMSE operators]\label{lemma:Lipschitz-LMMSE}
Consider the regularized LMMSE operator $\Lv^\lambda_\xv = \Cxy\pa{\Cyy + \lambda\Id}\inv$ and the sample one $\Lvapp^\lambda_\xv = \Cxyapp\pa{\Cyyapp+\lambda\Id}\inv$ with $\lambda \geq 0$ where both $\Cyy$ and $\Cyyapp$ are symmetric positive definite matrices. We define $\Delta_{\Cyy\inv} = \pa{\Cyy+\lambda\Id}\inv - \pa{\Cyyapp+\Id}\inv$, $\delta_{\Cyy\inv} = \norm{\Delta_{\Cyy\inv}}$ and $\delta_{\Cxy} = \norm{\Cxy - \Cxyapp}$. Then we have:
\begin{align*}
\Expect{}{\norm{\Lv_\xv^\lambda\yv - \widehat{\Lv}_\xv^\lambda\yv}^2} &=  \frac{1}{m}\Tr\pa{(\Lv^\lambda_\xv - \Lvapp^\lambda_\xv)\tp(\Lv^\lambda_\xv - \Lvapp^\lambda_\xv)\Cyy}\\
 &\leq \norm{\pa{\Cyy+\lambda}^{-1}} \delta_{\Cxy}^2 + \norm{\Cxyapp}^2 \norm{\Cyy\Delta_{\Cyy\inv}}\delta_{\Cyy\inv} + 2 \norm{\Cxyapp} \delta_{\Cxy} \delta_{\Cyy\inv}.
\end{align*}
\end{lemma}

\begin{proof}
We start by observing that
\begin{align*}
\Expect{\yv}{\norm{\Lv^\lambda_\xv\yv - \Lvapp^\lambda_\xv\yv}^2} = \Expect{\yv}{\yv\tp(\Lv^\lambda_\xv - \Lvapp^\lambda_\xv)\tp(\Lv^\lambda_\xv - \Lvapp^\lambda_\xv)\yv}.
\end{align*}
From there we use the trace trick and the linearity of expectation to get
\begin{align*}
\Expect{\yv}{\yv\tp(\Lv^\lambda_\xv - \Lvapp^\lambda_\xv)\tp(\Lv^\lambda_\xv - \Lvapp^\lambda_\xv)\yv} &= \Expect{\yv}{\Tr\pa{\yv\tp(\Lv^\lambda_\xv - \Lvapp^\lambda_\xv)\tp(\Lv^\lambda_\xv - \Lvapp^\lambda_\xv)\yv}}\\
&= \Expect{\yv}{\Tr\pa{(\Lv^\lambda_\xv - \Lvapp^\lambda_\xv)\tp(\Lv^\lambda_\xv - \Lvapp^\lambda_\xv)\yv\yv\tp}}\\
&= \Tr\pa{(\Lv^\lambda_\xv - \Lvapp^\lambda_\xv)\tp(\Lv^\lambda_\xv - \Lvapp^\lambda_\xv)\Cyy}.
\end{align*}
By using a mixed-product in $\pa{\Lv^\lambda_\xv - \Lvapp^\lambda_\xv}$ we can see that:
\begin{align*}
\begin{split}
\Tr\pa{(\Lv^\lambda_\xv - \Lvapp^\lambda_\xv)\tp(\Lv^\lambda_\xv - \Lvapp^\lambda_\xv)\Cyy} &= \Tr\Big(\pa{\pa{\Cyy+\lambda\Id}\inv\Delta_{\Cxy} + \Delta_{\Cyy\inv}\Cxyapp\tp} \\& \qquad \qquad \qquad \times\pa{\Delta_{\Cxy} \pa{\Cyy+\lambda\Id}\inv + \Cxyapp\Delta_{\Cyy\inv}}\Cyy\Big)
\end{split}\\
\begin{split}
&= \Tr\Big(\pa{\Cyy+\lambda\Id}\inv\Delta_{\Cxy}\tp\Delta_{\Cxy}\pa{\Cyy+\lambda\Id}\inv\Cyy + \Delta_{\Cyy\inv} \Cxyapp\tp\Cxyapp\Delta_{\Cyy\inv}\Cyy  \\ &  \qquad  +\Delta_{\Cyy\inv} \Cxyapp\tp\Delta_{\Cxy}\pa{\Cyy+\lambda\Id}\inv\Cyy + \pa{\Cyy+\lambda\Id}\inv \Delta_{\Cxy}\tp \Cxyapp \Delta_{\Cyy\inv}\Cyy\Big)
\end{split}
\end{align*}
We now use the linearity of the trace and its invariance to circular shift to obtain
\begin{align}\label{eq:Trace_sum_ineq_LMMSE}
\begin{split}
\Tr\pa{(\Lv^\lambda_\xv - \Lvapp^\lambda_\xv)\tp(\Lv^\lambda_\xv - \Lvapp^\lambda_\xv)\Cyy} &= \Tr\pa{\Cxyapp\tp\Cxyapp\Delta_{\Cyy\inv}\Cyy\Delta_{\Cyy\inv}} + 2\Tr\pa{\Delta_{\Cyy\inv} \Cxyapp\Delta_{\Cxy}\pa{\Cyy+\lambda\Id}\inv\Cyy} \\&+ \Tr\pa{\pa{\Cyy+\lambda\Id}\inv\Delta_{\Cxy}\tp\Delta_{\Cxy}\pa{\Cyy+\lambda\Id}\inv\Cyy}.
\end{split}
\end{align}
We can use the fact that $\Cyy$ is symmetric to observe that $\pa{\Cyy+\lambda\Id}\inv\Cyy \preceq \Id$. 

Knowing this, let us apply to~\eqref{eq:Trace_sum_ineq_LMMSE} the Von Neumann trace inequality alongside the fact that $\varsigma_i\pa{\Av\Bv}\leq \norm{\Av}\norm{\Bv}$ to get
\begin{align*}
\frac{1}{m}\Tr\pa{(\Lv^\lambda_\xv - \Lvapp^\lambda_\xv)\tp(\Lv^\lambda_\xv - \Lvapp^\lambda_\xv)\Cyy} &\leq \norm{\pa{\Cyy+\lambda}^{-1}} \delta_{\Cxy}^2 + \norm{\Cxyapp}^2 \norm{\Cyy\Delta_{\Cyy\inv}}\delta_{\Cyy\inv} + 2 \norm{\Cxyapp} \delta_{\Cxy} \delta_{\Cyy\inv}
\end{align*}
which complete the proof.
\end{proof}

We now state two classical lemmas to bound the errors that one can make when estimating empirical (cross)-covariance matrices. As these are existing results, we provide the proofs in the appendix.

\begin{lemma}[Sampling bound of empirical covariance]\label{lemma:error_emp_covariance}
Consider a finite sequence $\{\yv_k\}_{k=1}^{N}, N\in\mathbb{N}$ of independent random vectors with $\yv_k \in \R^m $ $ \forall k$. Assume that $\norm{\yv}^2 \leq \rho_\yv$ almost surely. Then, with probability at least $1-d$, we have for $\xi>0$ that
\begin{align*}
\norm{\frac{1}{N}\sum_{k=0}^N \pa{\yv_k - \theta_\yv}\pa{\yv_k - \theta_\yv} \tp - \Cyy} \leq \norm{\Cyy}\xi,
\end{align*}
when $N \geq \log\pa{\frac{2m}{d}}\frac{\rho_y\pa{6+4\xi}}{3\norm{\Cyy}\xi^2}$.
\end{lemma}

\begin{proof}
See Section~\ref{proof:error_emp_covariance}
\end{proof}

\begin{lemma}[Sampling bound of empirical cross-covariance]\label{lemma:error_emp_cross-covariance}
Consider two finite sequences $\{\xv_k\}_{k=1}^{N}, N\in\mathbb{N}$ and $\{\yv_k\}_{k=1}^{N}, N\in\mathbb{N}$ of independent random vectors with $\yv_k \in \R^m, \xv_k \in \R^n $  $ \forall k$. Assume that $\norm{\xv}^2 \leq \rho_x$ and $\norm{\yv}^2 \leq \rho_y$ almost surely. Then, with probability at least $1-d$, we have for $\xi>0$ that
\begin{align*}
\norm{\frac{1}{N}\sum_{k=0}^N\pa{\xv_k - \theta_\xv}\pa{\yv_k - \theta_\yv} \tp - \Cxy} \leq \norm{\Cxy}\xi,
\end{align*}
when $N \geq \log\pa{\frac{n+m}{d}}\frac{2\max\pa{\rho_\xv,\rho_\yv}\max\pa{\norm{\Cxx},\norm{\Cyy}}\pa{3 + 2\xi}}{3\min\pa{\norm{\Cxx},\norm{\Cyy}}^2\xi^2}$.
\end{lemma}

\begin{proof}
See Section~\ref{proof:error_emp_cross-covariance}
\end{proof}

Finally, we provide a lemma to bound the norm of the difference of two inverses of regularized operators.
\begin{lemma}[Difference of regularized inverses] \label{lemma:dif_regu_inverse_matrices}
 Let $\Mv \in \R^{m\times m}$ and $\widehat{\Mv} \in  \R^{m\times m}$ be symmetric matrices such that $\norm{\widehat{\Mv} - \Mv}\leq \xi$ with $\xi>0$. Let us also define $\gamma = \sigmin(\Mv)$ and $\lambda \in \mathbb{\R_+^*}$. If $\xi < \gamma + \lambda$, then:
\begin{align*}
\norm{\pa{\widehat{\Mv} + \lambda \Id}\inv - \pa{\Mv+\lambda\Id}\inv} \leq \frac{\xi}{\pa{\gamma + \lambda} - \xi}   \norm{\pa{\Mv + \lambda \Id}\inv} 
\end{align*}
\end{lemma}

\begin{proof}

 Using the fact that for invertible matrices $\Av\inv - \Bv\inv = \Av\inv\pa{\Bv - \Av}\Bv\inv$, we have that
\begin{align}\label{eq:D1_bound_regu_inverse}
\norm{\pa{\widehat{\Mv} + \lambda \Id}\inv - \pa{\Mv + \lambda \Id}\inv} &= \norm{\pa{\widehat{\Mv} + \lambda \Id}\inv\pa{\Mv - \widehat{\Mv}}\pa{\Mv + \lambda \Id}\inv}\nonumber\\
&\leq \norm{\pa{\widehat{\Mv} + \lambda \Id}\inv}\norm{\pa{\Mv + \lambda \Id}\inv}\xi.
\end{align}
The next step is to bound $\norm{\pa{\widehat{\Mv} + \lambda \Id}\inv}$ which we first reformulate as
\begin{align*}
\norm{\pa{\widehat{\Mv} + \lambda \Id}\inv} &= \norm{\pa{\Id - \pa{\Mv + \lambda \Id}\inv\pa{\Mv - \widehat{\Mv}}}\inv \pa{\Mv + \lambda \Id}\inv}.
\end{align*}
Now we can use the fact that $\xi \leq \gamma + \lambda$ which means that $\norm{\pa{\Mv + \lambda \Id}\inv\pa{\Mv - \widehat{\Mv}}} < 1$ and thus by Neumann series expansion, we obtain that
\begin{align*}
\norm{\pa{\Id - \pa{\Mv + \lambda \Id}\inv\pa{\Mv - \widehat{\Mv}}}\inv} = \frac{1}{1 - \frac{\xi}{\gamma + \lambda}}.
\end{align*}
This means that 
\begin{align}\label{eq:von_neumann_trick_inverse}
\norm{\pa{\widehat{\Mv} + \lambda \Id}\inv} \leq \norm{\pa{\Mv + \lambda \Id}\inv}\frac{1}{1 - \frac{\xi}{\gamma + \lambda}} = \frac{1}{\pa{\gamma + \lambda} - \xi}.
\end{align}
We can now combine this result with~\eqref{eq:D1_bound_regu_inverse} and get:
\begin{align}\label{eq:bound_final_D1_regu_inverse}
\norm{\pa{\widehat{\Mv} + \lambda \Id}\inv - \pa{\Mv + \lambda \Id}\inv} \leq \frac{\xi}{\pa{\gamma + \lambda} - \xi} \norm{\pa{\Mv + \lambda \Id}\inv}.
\end{align}

which concludes the proof.
\end{proof}

We can now state the proof of Theorem~\ref{th:main_sampling_bound}:
\subsubsection{Proof of Theorem~\ref{th:main_sampling_bound}}\label{proof:sampling_bound}
The starting point of the proof is Lemma~\ref{lemma:Lipschitz-LMMSE} that tells us that for a given linear operator of the form of $\Lvapp^\lambda_\xv$,  if we can bound $\norm{\Cyy\Delta_{\Cyy\inv}}$, $\norm{\Delta_{\Cyy\inv}}$ and $\norm{\Cxy - \Cxyapp}$ where $\Delta_{\Cyy\inv} = \pa{\Cyy+\lambda\Id}\inv - \pa{\Cyyapp+\Id}\inv$, then we can bound the error of such operator with respect to the LMMSE.

In our case, we consider as  operator the ``empirical'' LMMSE, that uses (paired) supervised data to estimate $\Cxy$ and $\Cyy$. More precisely, we will use $\Cxyapp = \frac{1}{N}\sum_{k=1}^N \xv_k\yv_k\tp$ and $\Cyyapp=\pa{\frac{1}{N}\sum_{k=1}^N \yv_k\yv_k\tp + \lambda\Id}$ as sample versions of $\Cxx$ and $\Cyy$ where we added the regularization $\lambda \Id$ to ensure invertibility. In order to obtain our bounds, we will rely on Lemma~\ref{lemma:error_emp_cross-covariance} and Lemma~\ref{lemma:error_emp_covariance} which provide sampling bounds for such empirical (cross-)covariance matrices. 

If we consider a number of samples $N$ such that $N > \log\pa{\frac{n+m}{d}}\frac{2\max\pa{\rho_\xv,\rho_\yv}\norm{\Cyy}\pa{3+2\xi}}{3\xi^2\norm{\Cxx}^2}$, we have by Lemma~\ref{lemma:error_emp_cross-covariance} and by Lemma~\ref{lemma:error_emp_covariance} that
\begin{align*}
\norm{\Cxy - \Cxyapp} \leq \norm{\Cxy}\xi \quad\quad \text{and} \quad\quad \norm{\frac{1}{N}\sum_{k=1}^N \yv_k\yv_k\tp - \Cyy} \leq \norm{\Cyy}\xi
\end{align*}
hold with probability $1-2d$ by using a union bound. We thus have our first intermediate bound.

In order to have a bound over $\norm{\Delta_{\Cyy\inv}}$, we use Lemma~\ref{lemma:dif_regu_inverse_matrices} with our previous result which gives us that
\begin{align*}
\norm{\pa{\Cyyapp + \lambda \Id}\inv - \pa{\Cyy+\lambda\Id}\inv} \leq  \norm{\pa{\Cyy + \lambda \Id}\inv} \frac{\xi}{\pa{\gamma + \lambda} - \xi}.
\end{align*}

For our last intermediate bound, by combining $\norm{\Cyy\pa{\Cyy + \lambda\Id}\inv}\prec \Id$ with the properties of two invertibles matrices $\Av\inv - \Bv\inv = \Av\inv\pa{\Bv - \Av}\Bv\inv$, we obtain that:
\begin{align*}
\norm{\Cyy\Delta_{\Cyy\inv}} &= \norm{\Cyy\pa{\Cyy + \lambda\Id}\inv\pa{\Cyyapp-\Cyy}\pa{\Cyyapp + \lambda\Id}\inv}\\
&\leq \xi \norm{\pa{\Cyyapp + \lambda\Id}\inv}.
\end{align*}
Now we can use the result of~\eqref{eq:von_neumann_trick_inverse} to obtain our final intermediate bound:
\begin{align*}
\norm{\Cyy\Delta_{\Cyy\inv}} \leq \frac{\xi}{\gamma + \lambda -\xi}
\end{align*}

Equipped with such bounds we can go back to Lemma~\ref{lemma:Lipschitz-LMMSE}  and use that $\norm{\Cxyapp} \leq \norm{\Cxy}(1+\xi)$ to obtain 
\begin{align*}
\Expect{\yv}{\frac{1}{m}\norm{\Lv^\lambda_\xv\yv - \Lvapp^\lambda_\xv\yv}^2} \leq \norm{\Cxy}^2\frac{\xi^2}{\gamma+\lambda}\pa{1 + \frac{\pa{1+\xi}^2}{\pa{\gamma+\lambda-\xi}^2} + \frac{2\pa{1+\xi}}{\pa{\gamma+\lambda-\xi}}}
\end{align*}
which is the claim of the theorem.

\subsubsection{Proof of Corollary~\ref{corollary:sampling_bound}}\label{proof:corollary_sampling}
Let us start from the results of Theorem~\ref{th:main_sampling_bound}. There the goal is to choose a $\xi$ that entails the desired rate. We will choose in our case $\xi=2\sqrt{\frac{K}{N}}$ which means that by our assumption on $\lambda$ our original bound $\gamma + \lambda >  \xi$ trivially holds. The next question is when does~\eqref{eq:bound_N_to_alpha} hold. For such $\xi$, plugging into~\eqref{eq:bound_N_to_alpha}:
\begin{align*}
N &> K\frac{3+2\xi}{\xi^2}\\
N &> \frac{N}{4}\pa{3 + \frac{\sqrt{K}}{\sqrt{N}}}\\
N &> N\pa{3/4 + \frac{\sqrt{K}}{4\sqrt{N}}}
\end{align*}
which holds when $N>K$.

We can now input this choice of $\xi$ in~\eqref{eq:sampling_bound_main}. We first observe that under our choice of $\lambda$ and $\xi$, we have that $\frac{\xi}{\gamma + \lambda} \leq 1/2$, which allows us to have the following chain of inequalities:
\begin{align*}
\frac{\xi+1}{\gamma+\lambda - \xi} &= \frac{1}{\gamma+\lambda}\frac{1+\xi}{1 - \frac{\xi}{\gamma+\lambda}}\\
&\leq \frac{1}{\gamma+\lambda}\pa{2+2\xi} \leq \frac{4}{\gamma+\lambda}
\end{align*}
where we used that when $N>K$ then $\xi\leq 1$. With this result in mind, we can replace directly  $\xi$ in~\eqref{eq:sampling_bound_main} to obtain
\begin{align*}
\Expect{\yv}{\norm{\Lv^\lambda_\xv\yv - \Lvapp^\lambda_\xv\yv}^2} \leq m\frac{9\norm{\Cxy}^2K}{\pa{\gamma+\lambda}N}\pa{1 + \frac{16}{\pa{\gamma+\lambda}^2} + \frac{8}{\gamma+\lambda}}
\end{align*}
which proves the claim.



\section{Conclusions and Outlook}


In this work, we provided the first theoretical characterization of Linear Minimum Mean Square Estimators (LMMSEs) for blind inverse problems and showed how uncertainty in the forward operator modifies the structure of the covariance matrix of the measurements compared to the non-blind setting. We derived finite-sample complexity bounds under an appropriate source condition by decomposing the expected error of the empirical LMMSE into the sum of an approximation error and a sampling error. Notably, the approximation error generalizes classical (i.e., non-blind) results by explicitly separating the contributions of measurement noise and operator uncertainty. In contrast, the sampling error bound does not show dependence  on the source condition and achieves the same 
$O(1/N)$ convergence rate as in the non-blind case, with $N$ denoting the sample size. All theoretical results are validated through illustrative numerical experiments.

Future work will focus on extending the presented theoretical analysis to nonlinear estimators tailored to the blind setting, including sparsity-promoting regularization schemes, in the spirit of recent work on learning optimal synthesis operators \cite{alberti2024learningsparsity}. Beyond the linear regime, we aim to further explore the role of (L)MMSE-type estimators in applied inverse problems, where they may provide a principled and computationally efficient alternative to commonly used MAP estimators, which are often suboptimal in practice \cite{nguyen2025diffusion, buskulic2026comparative}. To this end, we plan to investigate learning-based approaches—particularly neural network–based methods—to estimate the components of the LMMSE more efficiently. Building on recent theoretical advances in the non-blind setting, such approaches may indeed improve sample efficiency in high-dimensional regimes while preserving the approximation properties of LMMSE estimators and enabling tighter finite-sample guarantees.


\section*{Acknowledgments}

The work of N. Buskulic, L. Calatroni and S. Villa was supported by the funding received from the European Research Council (ERC) Starting project MALIN under the European Union’s Horizon Europe programme (grant agreement No. 101117133).
S. Villa acknowledges the support of the European Commission (grant TraDE-OPT 861137), of the European Research Council (grant SLING 819789), the US Air Force Office of Scientific Research (FA8655-22-1-7034), the Ministry of Education, University and Research (PRIN 202244A7YL project ``Gradient Flows and Non-Smooth Geometric Structures with Applications to Optimization and Machine Learning’’), and the project PNRR FAIR PE0000013 - SPOKE 10. The research by S. Villa. has been supported by the MIUR Excellence Department Project awarded to Dipartimento di Matematica, Università di Genova, CUP D33C23001110001. S. Villa is a member of the Gruppo Nazionale per l’Analisi Matematica, la Probabilità e le loro Applicazioni (GNAMPA) of the Istituto Nazionale di Alta Matematica (INdAM). 
This work represents only the view of the authors.
L. R. acknowledges the financial support of: the European Commission
(Horizon Europe grant ELIAS 101120237), the Ministry of Education, University and
Research (FARE grant ML4IP R205T7J2KP), the European Research Council (grant
SLING 819789), the US Air Force Office of Scientific Research (FA8655-22-1-7034), the
Ministry of Education, the grant BAC FAIR PE00000013 funded by the EU - NGEU
and the MIUR grant (PRIN 202244A7YL).
 The European Commission and the other organizations are not responsible for any use that may be made of the information it contains.


\bibliography{references}

\appendix
\section{Useful Lemmas for Theorem~\ref{th:main_sampling_bound}}

We detail here for completeness the proofs of lemmas that are used to prove Theorem~\ref{th:main_sampling_bound} for which we have either used existing results or adapted existing results.

\subsection{Proof of Lemma~\ref{lemma:error_emp_covariance}}\label{proof:error_emp_covariance}

\begin{lemma*}[\ref{lemma:error_emp_covariance} Sampling bound of empirical covariance]
Consider a finite sequence $\{\yv_k\}_{N\in \mathbb{N}}$ of independent random vectors with dimension $m$. Assume that $\norm{\yv}^2 \leq \rho_\yv$ almost surely. Then, with probability at least $1-d$, we have for $\xi>0$ that
\begin{align*}
\norm{\frac{1}{N}\sum_{k=0}^N \yv_k\yv_k \tp - \Cyy} \leq \norm{\Cyy}\xi,
\end{align*}
when $N \geq \log\pa{\frac{2m}{d}}\frac{\rho_y\pa{6+4\xi}}{3\norm{\Cyy}\xi^2}$.
\end{lemma*}

\begin{proof}
This is a very classical proof and builds on the matrix Bernstein inequality~\cite[Theorem 6.1.1]{tropp_introduction_2015}. The proof structure is found in \cite[Section 1.6.3]{tropp_introduction_2015}.
\end{proof}

\subsection{Proof of Lemma~\ref{lemma:error_emp_cross-covariance}}\label{proof:error_emp_cross-covariance}

\begin{lemma*}[\ref{lemma:error_emp_cross-covariance} Sampling bound of empirical cross-covariance]
Consider two finite sequences $\{\xv_k\}_{N\in \mathbb{N}}$ and $\{\yv_k\}_{N\in \mathbb{N}}$ of independent random vectors with dimension $n$ and $m$ respectively. Assume that $\norm{\xv}^2 \leq \rho_x$ and $\norm{\yv}^2 \leq \rho_y$ almost surely. Then, with probability at least $1-d$, we have for $\xi>0$ that
\begin{align*}
\norm{\frac{1}{N}\sum_{k=0}^N \pa{\xv_k - \theta_\xv}\pa{\yv_k - \theta_\yv} \tp - \Cxy} \leq \norm{\Cxy}\xi,
\end{align*}
when $N \geq \log\pa{\frac{n+m}{d}}\frac{2\max\pa{\rho_\xv,\rho_\yv}\max\pa{\norm{\Cxx},\norm{\Cyy}}\pa{3 + 2\xi}}{3\min\pa{\norm{\Cxx},\norm{\Cyy}}^2\xi^2}$.
\end{lemma*}

\begin{proof}
While the proof for covariance matrices is easily found, we add here for completeness the proof for the cross-covariance case, even if it can be found in other places. The proof relies on using a matrix Bernstein inequality~\cite[Theorem 6.1.1]{tropp_introduction_2015} to bound the covariance error.
\paragraph*{Bound the covariance term.} In order to bound our covariance term, we will use the matrix Bernstein inequality~\cite[Theorem 6.1.1]{tropp_introduction_2015} and we will use as summands $\Sv_k= \frac{1}{N} \pa{ \pa{\xv_k - \theta_\xv} \pa{\yv_k - \theta_\yv} \tp - \Cxy}$, as they are centered and independent since we use the true mean ($\theta_\xv$ and $\theta_\yv$) and not the empirical ones. 
 Let us start by bounding the norm of each summand in the following way:
\begin{align*}
\norm{\Sv_k} &= \frac{1}{N}\norm{\pa{\xv_k - \theta_\xv} \pa{\yv_k - \theta_\yv} \tp - \Cxy}\\
&\leq \frac{1}{N}\pa{\norm{\xv_k - \theta_\xv}\norm{\yv_k - \theta_\yv} + \norm{\Cxy}}\\
&\leq \frac{2\sqrt{\rho_\xv\rho_\yv}}{N} \leq \frac{2\max\pa{\rho_\xv,\rho_\yv}}{N}.
\end{align*}

We now turn our attention to the variance statistic of the sum of summands. As the summands are not symmetric we have to verify both $\norm{\sum_{k=1}^N \Expect{}{\Sv_k\Sv_k\tp}}$ and $\norm{\sum_{k=1}^N \Expect{}{\Sv_k\tp\Sv_k}}$. We start with the bound on $\norm{\sum_{k=1}^N \Expect{}{\Sv_k\Sv_k\tp}}$:
\begin{align*}
\norm{\sum_{k=1}^N \Expect{}{\Sv_k\Sv_k\tp}} &\leq \norm{\sum_{k=1}^N \frac{1}{N^2}\Expect{}{\pa{\pa{\xv_k - \theta_\xv} \pa{\yv_k - \theta_\yv} \tp - \Cxy}\pa{\pa{\xv_k - \theta_\xv} \pa{\yv_k - \theta_\yv} \tp - \Cxy} \tp}}\\
&\leq \frac{1}{N}\norm{\Expect{}{\norm{\yv_k - \theta_\yv}^2\pa{\xv_k - \theta_\xv}\pa{\xv_k - \theta_\xv}\tp} - \Cxy\Cxy\tp}\\
&\leq \frac{\rho_\yv}{N}\norm{\Cxx}
\end{align*}
where we used that $-\Cxy\Cxy\tp$ is negative semi-definite by definition and we could thus drop it in the norm. We now go to bounding $\norm{\sum_{k=1}^N \Expect{}{\Sv_k\tp\Sv_k}}$:
\begin{align*}
\norm{\sum_{k=1}^N \Expect{}{\Sv_k\tp\Sv_k}} &\leq \norm{\sum_{k=1}^N \frac{1}{N^2}\Expect{}{\pa{\pa{\xv_k - \theta_\xv} \pa{\yv_k - \theta_\yv} \tp - \Cxy}\tp\pa{\pa{\xv_k - \theta_\xv} \pa{\yv_k - \theta_\yv} \tp - \Cxy} }}\\
&\leq \frac{1}{N}\norm{\Expect{}{\norm{\xv_k - \theta_\xv}^2\pa{\yv_k - \theta_\yv}\pa{\yv_k - \theta_\yv}\tp} - \Cxy\tp\Cxy}\\
&\leq \frac{\rho_\xv}{N}\norm{\Cyy}.
\end{align*}
Where the last inequality hold because $-\Cxy\tp\Cxy$ is negative semi-definite. We can now combine the last two bounds to see that the variance statistic of our sum of summands is given by:
\begin{align*}
\max\pa{\norm{\sum_{k=1}^N \Expect{}{\Sv_k\Sv_k\tp}},\norm{\sum_{k=1}^N \Expect{}{\Sv_k\tp\Sv_k}}} \leq \frac{\max\pa{\rho_\xv,\rho_\yv}\max\pa{\norm{\Cxx},\norm{\Cyy}}}{N}.
\end{align*}

We are now in the position to calculate our matrix Bernstein concentration inequality which reads:
\begin{align*}
\mathbb{P}\pa{\norm{\frac{1}{N}\sum_{k=0}^N \xv_k\yv_k \tp - \Cxy} \geq t} \leq  \pa{n+m}\exp\pa{-\frac{3t^2N}{2\max\pa{\rho_\xv,\rho_\yv}\pa{3\max\pa{\norm{\Cxx},\norm{\Cyy}} + 2t}}}.
\end{align*}

Let us choose $t=\norm{\Cxy}\xi$ for some $\xi>0$. We first use that:
\begin{align*}
\min\pa{\norm{\Cxx},\norm{\Cyy}} \leq \norm{\Cxy} \leq \max\pa{\norm{\Cxx},\norm{\Cyy}}
\end{align*}
to obtain that our inequality is now:
\begin{align*}
\mathbb{P}\pa{\norm{\frac{1}{N}\sum_{k=0}^N \xv_k\yv_k \tp - \Cxy} \geq \norm{\Cxy}\xi} \leq  \pa{m+n}\exp\pa{-\frac{3\min\pa{\norm{\Cxx},\norm{\Cyy}}^2\xi^2N}{2\max\pa{\rho_\xv,\rho_\yv}\max\pa{\norm{\Cxx},\norm{\Cyy}}\pa{3 + 2\xi}}}.
\end{align*} 
The last step is to find some $N$ such that the right-hand side of this equation is equal to $d>0$. We obtain our result by the following chain of inequalities:
\begin{align*}
&\pa{m+n}\exp\pa{-\frac{3\min\pa{\norm{\Cxx},\norm{\Cyy}}^2\xi^2N}{2\max\pa{\rho_\xv,\rho_\yv}\max\pa{\norm{\Cxx},\norm{\Cyy}}\pa{3 + 2\xi}}} \leq d\\
&\frac{3\min\pa{\norm{\Cxx},\norm{\Cyy}}^2\xi^2N}{2\max\pa{\rho_\xv,\rho_\yv}\max\pa{\norm{\Cxx},\norm{\Cyy}}\pa{3 + 2\xi}} \geq \log\pa{\frac{n+m}{d}} \\
&N \geq \log\pa{\frac{n+m}{d}}\frac{2\max\pa{\rho_\xv,\rho_\yv}\max\pa{\norm{\Cxx},\norm{\Cyy}}\pa{3 + 2\xi}}{3\min\pa{\norm{\Cxx},\norm{\Cyy}}^2\xi^2}
\end{align*}
which gives the claim of the lemma.
\end{proof}

\section{Empirical covariance estimation}

We give a lemma that formalizes why using our ``hybrid'' empirical covariance strategy in~\eqref{eq:approx_LMMSE} that uses the true mean and not a sample estimation is acceptable given enough samples.

\begin{lemma}[Error between hybrid and empirical covariances]\label{lemma:covariance_error}
Consider two finite sequences $\{\xv_k\}_{k=1}^N$ and $\{\yv_k\}_{k=1}^N$ of independent random vectors with $\xv_k\in\R^n$ and $\yv_k\in\R^m$. Assume that $\norm{\xv_k - \theta_\xv}^2\leq\rho_\xv$ and  $\norm{\yv_k - \theta_\yv}^2\leq\rho_\yv$ almost surely. Let $\widehat{\Cv}_{\xv\yv} = \frac{1}{N}\sum_{k=1}^N\pa{\xv_k - \theta_\xv}\pa{\yv_k - \theta_\yv}\tp$ denote the empirical cross-covariance using true means and $\widetilde{\Cv}_{\xv\yv} = \frac{1}{N}\sum_{k=1}^N\pa{\xv_k - \widetilde{\xv}}\pa{\yv_k -  \widetilde{\yv}}\tp$ denote the fully empirical cross-covariance using sample means $\widetilde{\xv}$ and $\widetilde{\yv}$. Then with probability at least $1-2d$, we have:
\begin{align*}
\norm{\widehat{\Cv}_{\xv\yv} - \widetilde{\Cv}_{\xv\yv}}\lesssim (Nd)\inv.
\end{align*}
\end{lemma}

\begin{proof}
We first observe by expanding $\widetilde{\Cv}_{\xv\yv}$:
\[
\widetilde{\Cv}_{\xv\yv} = \frac{1}{N}\sum_{k=1}^N\pa{\xv_k - \theta_\xv - (\widetilde{\xv} - \theta_\xv)}\pa{\yv_k - \theta_\yv - (\widetilde{\yv} - \theta_\yv)}\tp.
\]
By distributing the terms and using that $\frac{1}{N}\sum_{k=1}^N(\xv_k -  \theta_\xv)=\widetilde{\xv} - \theta_\xv$ (and symmetrically for $\yv$), the cross term simplify, yielding:
\[
\widetilde{\Cv}_{\xv\yv} - \widehat{\Cv}_{\xv\yv} = -\pa{\widetilde{\xv} - \theta_\xv}\pa{\widetilde{\yv} - \theta_\yv}\tp.
\]
Taking the operator norm on both sides and utilizing that the norm of a rank-1 matrix $\av\bv\tp$ is $\norm{\av}\norm{\bv}$ we obtain
\[
\norm{\widetilde{\Cv}_{\xv\yv} - \widehat{\Cv}_{\xv\yv}} = \norm{\widetilde{\xv} - \theta_\xv}\norm{\widetilde{\yv} - \theta_\yv}.
\]
To bound the norm of these mean errors we consider the expected square norm $\Expect{}{\norm{\widetilde{\xv} - \theta_\xv}^2}$. Because the samples are independent and centered around $\theta_\xv$, the cross-term in the expansion of the square norm have an expectation of zero. Thus this expectation can be written:
\[
\Expect{}{\norm{\widetilde{\xv} - \theta_\xv}^2} = \frac{1}{N^2}\sum_{k=1}^N \Expect{}{\norm{\xv_k - \theta_\xv}^2} \leq \frac{\rho_\xv}{N}.
\]
By Markov's inequality we have
\[
\mathbb{P}\pa{\norm{\widetilde{\xv} - \theta_\xv}^2 \geq t^2} \leq \frac{\Expect{}{\norm{\widetilde{\xv} - \theta_\xv}^2}}{t^2} \leq \frac{\rho_\xv}{Nt^2}.
\]
By choosing $t=\sqrt{\frac{\rho_\xv}{Nd}}$, we have that with probability at least $1-d$:
\[
\norm{\widetilde{\xv} - \theta_\xv} \leq \sqrt{\frac{\rho_\xv}{Nd}}.
\]
Thus by the same logic we have that $\norm{\widetilde{\yv} - \theta_\yv} \leq \sqrt{\frac{\rho_\yv}{Nd}}$. Using a union bound we have that with probability at least $1 - 2d$:
\[
\norm{\widetilde{\Cv}_{\xv\yv} - \widehat{\Cv}_{\xv\yv}} \leq \frac{\sqrt{\rho_\xv\rho_\yv}}{Nd},
\]
which proves the claim of the lemma.
\end{proof}

\end{document}